\documentclass[11pt]{article}
\usepackage{fullpage}
\usepackage{wustyle}

\usepackage{braket}
\usepackage[top=0.9in,bottom=1.1in,left=1.1in,right=1.1in]{geometry}
\usepackage{authblk}

\newcommand{\nb}[0]{\nabla}
\newcommand{\al}[0]{\alpha}

\title{Improved Analysis of Score-based Generative Modeling: User-Friendly Bounds under Minimal Smoothness Assumptions}

\author{Hongrui Chen\thanks{Peking University, \href{mailto:hongrui_chen@pku.edu.cn}{hongrui\_chen@pku.edu.cn} }
\quad
Holden Lee\thanks{Johns Hopkins University, \href{hlee283@jhu.edu}{hlee283@jhu.edu}}
\quad
Jianfeng Lu\thanks{Duke University, \href{jianfeng@math.duke.edu}{jianfeng@math.duke.edu}}
}

\date{\today}

\begin{document}

\maketitle
\begin{abstract}
We give an improved theoretical analysis of score-based generative modeling. 
Under a score estimate with small $L^2$ error (averaged across timesteps), 
we provide efficient convergence guarantees 
for any data distribution with second-order moment, by either employing early stopping or assuming smoothness condition on the score function of the data distribution. Our result does not rely on any log-concavity or functional inequality assumption and has a logarithmic dependence on the smoothness. In particular, we show that under only a finite second moment condition, approximating the following in reverse KL divergence in $\epsilon$-accuracy can be done in $\tilde O\left(\frac{d \log (1/\delta)}{\epsilon}\right)$ steps: 1) the variance-$\delta$ Gaussian perturbation of any data distribution; 2) data distributions with $1/\delta$-smooth score functions. Our analysis also provides a quantitative comparison between different discrete approximations and may guide the choice of discretization points in practice.
\end{abstract}

\section{Introduction}
Generative modeling is one of the central tasks in machine learning, which aims to learn a probability distribution from data and generate data from the learned distribution. Score-based generative modeling (SGM) has achieved state-of-art performance in data generation tasks \cite{SGMsong,SGMSDEsong,Song2021MaximumLT,Dhariwal2021DiffusionMB}, surpassing other models like generative adversarial networks (GAN) \cite{Goodfellow2014GenerativeAN}, normalizing flows \cite{JimenezRezende2015VariationalIW}, variational autoencoders \cite{Kingma2014AutoEncodingVB}, and energy-based models \cite{Zhao2016EnergybasedGA}. Due to the impressive sample quality, SGM has great potential in various applications, including computer vision \cite{Dhariwal2021DiffusionMB,Rombach2021HighResolutionIS}, natural language processing \cite{Austin2021StructuredDD}, inverse problems \cite{Song2022SolvingIP,Chung2021ComeCloserDiffuseFasterAC}, molecular graph modeling \cite{Shi2021LearningGF,Gnaneshwar2022ScoreBasedGM}, reinforcement learning \cite{Wang2022DiffusionPA}, and solving high-dimensional PDEs \cite{Boffi2022ProbabilityFS}.

The key idea of SGM is to use a forward process to diffuse the data distribution to some prior (often the standard Gaussian), and learn a backward process to transform the prior to the data distribution by estimating the score functions of the forward diffusion process. Such a procedure provides an expressive and efficient way to model high-dimensional distributions for two reasons: 1) It is easy to construct a forward process that converges fast to the Gaussian, no matter how complex the data distribution is. For example, the Ornstein-Uhlenbeck (OU) process has stationary distribution equal to the standard Gaussian and converges rapidly. 
2) Several scalable score matching methods such as denoising score matching \cite{Vincent2011ACB} and sliced score matching \cite{Song2019SlicedSM} allow us to learn the score function 
for use by the backward process.

While SGM has achieved great success in practice, theoretical understanding of the power of SGM is far from complete. Recent works \cite{convergencescore2, SamplingEasy} established that when an accurate score estimator is given, SGM can sample from general distributions with polynomial complexity and without requiring structural assumptions such as log-concavity or functional inequalities. (By \emph{polynomial complexity} we mean that the running time is polynomial and the final error depends polynomially on the score estimation error and other parameters.) This is surprising in the sampling context, as it implies a sharp contrast between SGM and 
sampling dynamics with gradient flow structure (such as Langevin dynamics), where convergence rates depend crucially on the structure of the data distribution. In this paper, we further establish the effectiveness of SGM by showing that convergence with reasonable rates requires very weak smoothness conditions. Indeed, we obtain a logarithmic dependence on the smoothness, or no dependence when comparing against a slightly perturbed data distribution.

\subsection{Background and Our Setting}

\paragraph*{General Framework.}
Let $P$ be the data distribution on $\RR^d$. Given data $\{x_i\}_{i=1}^n$ sampled from the data distribution $P$, the first step of SGM involves gradually transforming the data distribution into white noise by a forward SDE:
\begin{align}
    \d x_{t} = f(x_t,t) \d t     + g(t) \d w_t,\, x_0 \sim P,\,0\leq t \leq T.  \label{forward}
\end{align}
We use $p_t(x)$ to denote the density of $x_t$. In particular, $p_T$ is close to the white noise distribution $\cN(0,I_d)$. Then $x_t$ also satisfies the reverse SDE 
\begin{align}
    \d {x}_t = \left(f({x}_t, t) - g(t)^2 \nabla \log p_{t}({x}_t) \right)\d t + g(t) \d \tilde{w}_t\label{backward},
\end{align}
where $\tilde{w}_t$ is backward Brownian motion \cite{Anderson1982ReversetimeDE}. For convenience, we rewrite the reverse SDE \eqref{backward} in a forward version by switching time direction $t \rightarrow T -t$:
\begin{align}
     \d \tilde{x}_t = \left(-f(\tilde{x}_t, T-t) + g(T-t)^2 \nabla \log p_{T-t}(\tilde{x}_t) \right)\d t 
     + g(T-t) \d {w}_t  \label{backward2},
\end{align}
where $w_t$ is the usual (forward) Brownian motion. 
The process $(\tilde{x}_t)_{0\leq t \leq T}$ transforms noise into samples from $P$, which accomplishes the goal of generative modeling. 

However, we cannot directly simulate \eqref{backward2} since the score function $\nabla \log p_t$ is not available. Thus we learn the score function $\nabla \log p_t$ from the noisy data. First, we parameterize the score function within a function class such as that of neural networks, $s_\theta(x,t)$. Then we optimize one of the score-matching objectives (denosing score matching \cite{Vincent2011ACB} is often used; see appendix \ref{DSM} for details), from which we obtain a score estimator $s_\theta$ such that the $L^2$ score estimation error 
\begin{align*}
    \EE_{p_t} \|s_\theta(x,t) - \nabla \log p_t(x) \|^2
\end{align*}
is small. Using the estimated score, we can generate samples from an approximation of the reverse SDE starting from the prior distribution: 
\begin{multline}
    \d y_t = \left(-f(y_t,T-t) + g(T-t)^2 s_\theta(y_t,T-t) \right)\d t + g(T-t) \d w_t,\,y_0 \sim p_{\textup{prior}},\, 0\leq t \leq T. \label{continuous}
\end{multline}

\paragraph*{The Choice of Forward Process.}
We focus on the case $f(x,t) = -\frac{1}{2}x,\,g(t) \equiv 1$. The choice of $f(x,t)$ matches the choice in the original paper \cite{SGMSDEsong}, though our analysis may be adapted for some other choices of drift terms; the choice of constant variance function does not cause any loss of generality since the changing the variance function is equivalent to rescaling time (when $f$ does not depend on $t$). In this case, the forward process becomes the Ornstein-Uhlenbeck process, which has an explicit conditional density: 
$$ x_t|x_0  \sim \cN \left(e^{-\frac{1}{2}t}x_0, \left(1-e^{-t}\right)I_d \right).$$
Moreover, the Ornstein-Uhlenbeck process converges exponentially to the standard Gaussian distribution: 
\begin{align*}
    \mathrm{KL}(p_t \| \cN(0,1)) \leq e^{-t} \mathrm{KL}(p_0 \| \cN(0,1)).
\end{align*}


\paragraph*{Time Discretization.}
In practice, we need to use a discrete-time approximation for the sampling dynamics \eqref{continuous}. Let $ \delta = t_0 \leq t_1 \leq \cdots \leq t_N = T$ be the discretization points, where $\delta = 0$ for the normal setting and $\delta>0$ for the early-stopping setting.  For the $k$-th discretization step ($1\leq k \leq N$), we denote $h_{k} := t_k - t_{k-1}$ as the step size. We will compare different choices of discretization points and identify the optimal choice in different settings. 

Let $t_k' = T-t_{N-k}$ be the corresponding discretization points in the reverse SDE. 
We consider two types of discretization schemes, which are widely used in existing work.
\begin{itemize}
\item 
The Euler-Maruyama scheme:
\begin{align}
\mathrm{d} \hat{y}_t=\left[\frac{1}{2} \hat{y}_{t_k^{\prime}}+s_\theta\left(\hat{y}_{t_k^{\prime}}, T-t_k^{\prime}\right)\right] \mathrm{d} t+\mathrm{d} w_t, t \in\left[t_k^{\prime}, t_{k+1}^{\prime}\right], \label{EM-discrete}
\end{align}
for $k=0,1, \ldots, N-1$.
\item The expontential integrator scheme \cite{Song2021DenoisingDI,Zhang2022FastSO}: by using the semi-linear structure of \eqref{backward}, we discretize only in the nonlinear term and retain the continuous dynamics arising from the linear term:
   \begin{align}    \d \hat{y}_t = \left[\frac{1}{2}\hat{y}_t +  s_\theta(\hat{y}_{t_k'}, T-t_k')\right] \d t + \d w_t,\, t \in [t_k', t_{k+1}']  \label{semi-discrete}
   \end{align}
   for $k=0,\ldots,N-1$,
   which is solved explicitly by
   \begin{align*}
  & \hat{y}_{t_{k+1}'} = e^{\frac{1}{2}\left(t_{k+1}'  - t_k'
    \right)}
    \hat{y}_{t_k'}
   + 
   2
   \left(e^{\frac{1}{2}\left(t_{k+1}'  - t_k'
    \right)}-1\right)
    s_\theta(\hat{y}_{t_k'},T-t_k') 
    + 
    \sqrt{e^{t_{k+1}'  - t_k'
    }-1}\cdot \eta_k,
   \end{align*}
 where $ \eta_k\sim \mathcal N(0,I_d)$.
\end{itemize}

\subsection{Related Work}
We highlight two recent papers \cite{SamplingEasy,convergencescore2}. Both papers provide convergence guarantees with polynomial complexity without relying on any structural assumptions on the data distribution such as log-concavity or a functional inequality. In particular, the analysis of \cite{SamplingEasy} is based on the Girsanov change of measure framework and the authors consider the following two settings: 1) The score functions in the whole trajectory of the forward process satisfy the Lipschitz condition with a uniform Lipschitz constant. 2) The data distribution has bounded support. Although the smoothness condition on the forward process seems mild, it may be hard to check whether the uniform bound for the Lipschitz constants scales polynomially w.r.t. the dimension $d$. In fact, this is a property of the whole process, related to tail bounds of the data distribution. The work \cite{convergencescore2} alternatively
uses the idea of excluding bad sets in order to reduce to the setting of an $L^\infty$-accurate score estimator. This results in a worse dependence on the problem parameters; however, they do relax the smoothness condition on the whole trajectory to one on only the data distribution, and the bounded support assumption to sufficient tail decay.

Many other works have provided convergence analyses, but do not achieve polynomial complexity except in restricted settings, for example relying on functional inequalities (thus precluding multi-modal distributions) \cite{Block2020GenerativeMW,convergence-score,inexactLangevin}, manifold hypotheses \cite{Bortoli2022ConvergenceOD}, or $L^\infty$-accurate score estimates \cite{Bortoli2021DiffusionSB}.
In the setting where only an $L^2$-accurate score estimate of the data distribution is given, \cite{statistic-efficiency} 
give a statistical lower bound which 
shows it is in general impossible to accurately sample the distribution. 
This highlights the fact that having score estimates for multiple distributions---e.g., the data distribution with different amounts of noise added---is necessary for efficient sampling; this is done in practice and in our analysis. In a different direction, SGM is also related to recent work on algorithmic stochastic localization \cite{alaoui2022sampling}, in which for the spin glass models under consideration, the score function (i.e., the posterior mean) can be accurately estimated using approximate message passing. 

\subsection{Our Contributions}
In this paper, we quantitatively show that an $L^2$-accurate score estimator is enough to guarantee that the sampling dynamics \eqref{EM-discrete}, \eqref{semi-discrete} result in a distribution close to the data distribution in various regimes. Our results combine the advantages of \cite{SamplingEasy,convergencescore2}: under weak assumptions on the data distribution and the score estimator, we provide a concise analysis and refined guarantees for the convergence of SGM under several settings, described below and summarized in Table~\ref{table}.

\begin{table*}[h!]
\caption{Suppose $p_0$ has bounded 2nd moment $M_2$ and average $L^2$ score error is at most $\epsilon_0^2$. Guarantees for DDPM hold under the following smoothness assumptions, listed in order of decreasing strength. Note the 2nd bound also holds under the 1st assumption, but trades off dependence on $d$ and $L$. \cite{SamplingEasy} obtain TV guarantees, which are weaker by Pinsker's inequality. 
}
\label{table}
\begin{center}
\begin{small}
\begin{tabular}{l|c|c|c}
Assumption & Error guarantee & Steps to get $\tilde{O}(\epsilon_0^2)$ error & Theorem \\
\hline
$\forall t,\,\nabla \log p_t$ $L$-Lipschitz    & $\mathrm{KL}(p_0\|\hat{q}_T)$   &  $\tilde O\left(\frac{dL^2}{\epsilon_0^2}\right) $ & Theorem~\ref{smoothbound}\\
& $\TV(p_0, \hat q_T)^2$ & $\tilde O\pf{(d\vee M_2)L^2}{\ep_0^2}$ &\cite[Theorem 2]{SamplingEasy} \\
\hline
$\nabla \log p_0$ $L$-Lipschitz  &$\KL(p_0 \| \hat{q}_T)  $  &   $ \tilde O\left( \frac{d^2 \log^2 L}{\epsilon_0^2}\right)$ & Theorem~\ref{truncating} \\
\hline 
None & $\KL(p_\delta \| \hat{q}_{T-\delta}) $ & $\tilde O\left(\frac{d^2\log^2(1/\delta)}{\epsilon_0^2} \right) $ & Theorem~\ref{general} \\
Supported on $B_R(0)$
& $\TV(p_\delta, \hat q_{T-\delta})^2$
& $\tilde O \left(\frac{(d\vee M_2) R^4}{\ep_{0}^2\delta^4}\right)$ & 
\cite[Thm. 2 + Lem. 16]{SamplingEasy}\\
\hline 
\end{tabular}
\end{small}
\end{center}
\vskip -0.1in
\end{table*}
\paragraph{Smooth setting.}
Revisiting the setting where the Lipshitz constant of $\nabla \log p_t,\,0\leq t \leq T$ is uniformly bounded (the trajectory-smooth setting), we provide three refinements compared to \cite{SamplingEasy}: 1) We sidestep the technical issue of checking Novikov's condition and provide a reverse KL divergence guarantee, which is stronger than a TV guarantee. 2) For the exponential integrator scheme, the number of steps dependends logarithmically rather than polynomially on the second moment. 3) We do not assume the data distribution has finite KL divergence wrt the standard Gaussian.
 \paragraph{Non-smooth setting.} 
 We provide convergence guarantees for sampling from any distribution with bounded second-order moment, without any structural assumption or smoothness condition. In particular, for any small constant $\delta>0$, we show that running the sampling dynamics \eqref{semi-discrete} with appropriate early stopping and decreasing step size results in a distribution close to $p_{\delta}$, using a high-probability bound on the Hessian matrix $\nabla^2 \log p_t$ and a change-of-measure argument. Comparing to the early stopping result in \cite{SamplingEasy}, the use of a high-probability rather than uniform bound on the Hessian removes the bounded support assumption and induces a significantly tighter dependence on the problem parameters. Quantitatively, to obtain a bound of $\ep_{\TV}$ in TV-distance to $p_\de$, when the data distribution is supported on a ball of radius $R$, \cite{SamplingEasy} require $\tilde \Theta \left(\frac{dR^4}{\ep_{\TV}^2\delta^4}\right)$ steps, 
 while we consider a distribution with second moment bounded by $M_2$ and only require $\tilde \Theta \left(\frac{d^2}{\ep_{\TV}^2} 
 \log^2 \frac{M_2d}{\delta}
 \right)$ steps (typically, $R\asymp \sqrt d$).
 We have no dependence on $R$, and our dependence on $\delta$ and $M_2$ is logarithmic instead of polynomial.

 By adding an extra truncation step on the algorithm, we also obtain a pure Wasserstein bound depending on the tail decay of the data distribution, significantly improving the prior result \cite[Theorem 2.2]{convergence-score}.

%
 \paragraph{Smooth $p_0$ only.} 
Finally, we consider the intermediate assumption of smoothness of $\nabla \log p_0$, rather than the whole forward process as in~\cite{SamplingEasy}. 
In this case, we can bound discretization error in the low-noise regime so that early stopping is not required. 
We combine the smooth and non-smooth analyses to bound the number of steps logarithmically in $L$, the Lipschitz constant of $\nabla \log p_0$.

Furthermore, we analyze difference choices of discretization schemes and step-size schedules (equivalently, different variance functions). This may help guide the practical implementation of SGM.

\subsection{Notations}
\paragraph*{General Notations.}
Let $d$ be the dimension of the data, and $\gamma_d$ be the density of standard Gaussian measure $\cN(0,I_d)$. $\|\cdot\|$ denotes the $\ell^2$ norm for vectors or the spectral norm for matrices, and  $\|\cdot \|_F $ denotes the Frobenius norm of matrices. For a random variable $X$, the sub-exponential and sub-gaussian norms are defined by
\begin{align*}
    \|X\|_{\psi_k} :&= \inf\{t>0: \EE \exp\left(|X|^k/t \right) \leq 2 \},\, k=1,2.
\end{align*}
For random vectors, we denote $\|\cdot\|_{\psi_k} := \|\|\cdot\|\|_{\psi_k}  $. 
We use $ x \asymp y$ if there exist absolute constants $C_1, C_2 > 0$ such that $C_1y \leq x \leq C_2y$. Write $x \lesssim y $ to mean $x \leq Cy$ for an absolute constant $C > 0$, and define $x \gtrsim y$ analogously.
\paragraph*{Notations for the Forward Process.}
Let $P$ be the data distribution and $p_0$ be its density (if it exists). For $0< t \leq T$, let $p_t$ be the density of $x_t$ defined in the forward process \eqref{forward} with $f(t,x) = \frac{1}{2}g(t)^2x_t$. Define $\sigma_t$ as the conditional variance of $x_t$ given $x_0$, i.e.,
\begin{align*}
      \sigma_t^2 := 1- e^{-t}.
\end{align*}
For any $ 0\leq t \leq s \leq T$, let 
$$\alpha_{t,s} := 
e^{-\frac{1}{2}(s-t)},\quad \alpha_t :=\alpha_{0,t}$$
gives the scaling between times $t$ and $s$: $\EE[x_s|x_t] = \alpha_{t,s} x_t$.

\paragraph*{Notations for Reverse Processes.}
Let $s(x,t)$ be the estimated score function. The reverse processes arising in our setting are defined as follows:
\begin{itemize}
    \item Let $\tilde{x}_t$ be the the reverse process of $(x_t)_{0 \leq t \leq T}$, which is driven by the SDE
    \begin{align*}
      \d   \tilde{x}_t = \left(\frac{1}{2}\tilde{x}_t  + \nabla \log p_{T-t}(\tilde{x}_t) \right)\d t  + \d w_t,\, \tilde{x}_0 \sim p_T
    \end{align*}
    Then the law of $(\tilde{x}_t)_{0 \leq t \leq T}$ is identical to the law of $({x}_{T-t})_{0 \leq t \leq T}$. We use $\tilde{p}_t$ to denote the density of $\tilde{x}_t$.
     \item Let $\hat{y}_t$ be the discrete approximation of $y_t$ defined in \eqref{EM-discrete} or \eqref{semi-discrete} starting from $\hat{y}_0 \sim \cN(0,I_d)$. We use $\hat{q}_t$ to denote the density of $\hat{y}_{t}$.
\end{itemize}

\section{Main Results}  \label{Main-result}
We first 
consider the trajectory smoothness assumption, where we strengthen the result of \cite{SamplingEasy}.
Then, we state our results for more general settings in various regimes.

All the results rely on $L^2$-accuracy of the score estimator:
\begin{assumption}\label{as1}
The learned score function $s(x,t)$ satisfies for any $1\leq k \leq N$,
\begin{align} \label{score-error}
\frac{1}{T}\sum_{k=1}^N h_k\EE_{p_{t_k}} \|\nabla \log p_{t_k}(x) - s(x,t_k) \|^2 \leq \epsilon_0^2.
\end{align}
\end{assumption}
\begin{remark}\label{r:avg}
Because this is a weighted average of score estimation errors on the discretization points, it can be satisfied even if the error diverges as $t\to 0$.
This is useful because simply based on the size of the gradient, we can expect the error to scale as 
$\EE_{p_{t_k}} \|\nabla \log p_{t_k}(x) - s(x,t_k) \|^2 \lesssim \frac{\epsilon^2}{\sigma_{t_k}^2} $, where $\si_{t}^2\sim t$ as $t\to 0$. 
The calculation $\int_{t_1}^1\frac 1t\,dt = \log(1/t_1)$ tells us we can take $\ep_0^2=O(\ep^2 \log(1/t_1))$. 
See Appendix \ref{DSM} for details.   
\end{remark}


  \begin{assumption}\label{as4}
      The data distribution has a bounded second moment: $M_2 := \EE_{P} \|x\|^2 < \infty $.
   \end{assumption}

\subsection{Analysis under the Trajectory Smoothness Condition} \label{recall}

First, we improve result of \cite{SamplingEasy} for the trajectory-smooth setting, weakening the assumptions and strengthening the conclusion.

   
 \begin{assumption} \label{as2}
  For any $0\leq t \leq T$, $\nabla \log p_t$ is $L$-Lipschitz on $\RR^d$.
  \end{assumption}

 \begin{theorem}  \label{smoothbound}
Suppose that Assumptions \ref{as1},\ref{as4},\ref{as2} hold. If $L \geq 1 $, $h_k \leq 1 $ for $k=1,\ldots,N$ and $T \geq 1$, using uniform discretization points yields the followings
\begin{itemize}
\item Using exponential integrator scheme \eqref{semi-discrete}, we have
$$
           \mathrm{KL}(p_0 \| \hat{q}_T) \lesssim (M_2+d)e^{-T}+ T\epsilon_{0}^2 + \frac{dT^2L^2}{N}.$$
 In particular, choosing $T=\log \left(\frac{M_2+d}{\ep_0^2}\right)$ and $N=\Theta\left(\frac{d T^2 L^2}{\ep_0^2}\right)$ makes this $\widetilde{O}\left(\ep_0^2\right)$.
 \item Using the Euler-Maruyama scheme \eqref{EM-discrete}, we have
 $$ \mathrm{KL}(p_0 \| \hat{q}_T) \lesssim (M_2+d)e^{-T}+ T\epsilon_{0}^2 + \frac{dT^2L^2}{N} + \frac{T^3M_2}{N^2}. $$
\end{itemize}
\end{theorem}
For the exponential integrator, the error consists of three parts: 
 the error of the forward process, the score matching error, and the discretization error, detailed in Section~\ref{s:sketch}.
 \begin{remark}
 \begin{itemize}
 \item 
 The extra conditions on $L,h_k,T$ in the above theorem are introduced to present the result more concisely, and are not a limitation of the analysis. 
 \item Comparing to the exponential integrator scheme, the Euler-Maruyama scheme causes an
additional high-order discretization error term related to the second-order moment of the
data distribution. This implies a separation between the exponential integrator scheme
and the Euler-Maruyama scheme: the error of the exponential integrator scheme scales
logarithmically in the second moment of the data distribution (as it suffices for $T$ to increase by
$O(\log M_2)$), while the error of the Euler-Maruyama scheme scales linearly.
\item Rather than TV distance guarantees given in \cite{SamplingEasy}, 
we obtain (reverse) KL divergence guarantees which are stronger by Pinsker's inequality and nontrivial even when $\ep_{\KL}\ge 1$.
 \end{itemize}
 \end{remark}

 
 
\paragraph*{Discussion for Lipschitzness Assumption \ref{as2}.}
 Though Assumption \ref{as2} seems mild, it is hard to check whether the Lipschitz constant of the score function is bounded uniformly by a constant $L = O(\mathrm{poly}(d))$ throughout the entire process. In the log-concave setting, the smoothness of $\nabla \log p_0$ implies the smoothness of $\nabla \log p_t$ \cite[Lemma 28]{lee2021universal}.
 However, for non-log-concave distributions such as multi-modal distributions, 
 this can be difficult to check, and 
 may depend on the tail behavior of the data distribution. Our aim 
 in this work is to relax such smoothness assumptions. 
 
 \subsection{Results for General Distributions with Early Stopping} \label{sub-general}
We now consider the most general setting: we provide convergence guarantees for any distribution that has a bounded second-order moment, without introducing any structural assumptions or smoothness conditions. Hence, our results are applicable to the case that the score function is non-smooth or even not well defined, like distributions supported on a low-dimensional manifold. 

Due to our weak assumptions, the backward process \eqref{backward} may have very bad properties when $t$ is close to $0$, so we need to employ early stopping. For any small constant $\delta > 0$, we show that running the sampling dynamics \eqref{semi-discrete} for time $T-\delta$ will result in a distribution close to $p_{\delta}$ in KL divergence. Note that in general, it is impossible to obtain  KL or TV closeness to $P$ as this requires matching exactly the support of $P$.

We provide the convergence bound for general discretization and further quantify the bound for several specific choices.

\begin{theorem} \label{general}
There is a universal constant $K$ such that the following hold. 
Suppose that Assumptions \ref{as1} and \ref{as4} hold and the step sizes satisfy
\begin{align} \label{condition}
\frac{h_k}{\sigma_{t_{k-1}}^2} \le \frac{1}{Kd},\quad k=1,\ldots,N.
\end{align}
Define $\Pi:= \sum_{k=1}^N \frac{h_k^2}{\sigma_{t_{k-1}}^4}$.
For $T \geq 2, \delta \leq \frac{1}{2} $, the exponential integrator scheme \eqref{semi-discrete} with early stopping result in a distribution $\hat{q}_{T-\delta}$ such that 
    \begin{align}
        \mathrm{KL}(p_{\delta}\|\hat{q}_{T-\delta}) \lesssim  (d+M_2)\exp(-T) + T\epsilon_{0}^2  + d^2\Pi.
        \label{e:kl-ineq}
    \end{align}
In particular, for exponentially decreasing step size $h_k= c\min\{t_k,1\}$, where $c\le \frac1{Kd}$ (or, equivalently $\frac{{\log\prc{\de} + T}}{N} \le \frac{1}{Kd} $), then \eqref{condition} holds and
 \begin{align*}
 \Pi \lesssim \frac{\pa{\log\prc{\de} + T}^2}{N}.
 \end{align*}
 Choosing $T=\log \left(\frac{M_2+d}{\ep_0^2}\right), N=\Theta\left(\frac{\left(\log \left(\frac{1}{\delta}\right)+T\right)^2 d^2}{\ep_0^2}\right)$ makes this $\widetilde{O}\left(\ep_0^2\right)$.

 In addition, for Euler-Maruyama scheme \eqref{EM-discrete}, the same bounds hold with an additional term $M_2\sum_{k=1}^N h_k^3 $
term in the right hand side of \eqref{e:kl-ineq}.
 
 
\end{theorem}

\begin{remark}
\begin{itemize}
    \item The technical condition \eqref{condition} is required for the change-of-measure argument in Lemma \ref{noLip}. 
    \item By rescaling time, choosing constant variance function $g\equiv 1$ and exponentially decreasing step size is equivalent to choosing exponential $g$ and constant step size. We state the theorem with constant $g$ for convenience (with an exponential choice of $g$, we would only reach the data distribution $P$ at time $t=-\infty$). 
\end{itemize}
\end{remark}

The key difficulty in analyzing general distributions is that the discretization error is hard to control without the Lipschitz condition on $\nabla \log p_t$. Our approach is to use a high-probability bound for the Hessian matrix $\nabla^2 \log p_t$ with a change of measure. This approach works well for constant-order $t$, while in the low-noise regime the bound will explode as $t$ tends to 0. We overcome the blow-up of discretization error by early stopping. 


\paragraph*{Discussion on the Choice of Discretization Points.}
When $t$ goes to 0, the regularity of $\nabla \log p_t$ becomes worse so slowing down the SDE leads to a smaller discretization error. In the result of Theorem~\ref{general}, the term $\Pi = \sum_{k=1}^N \frac{h_k^2}{\sigma_{t_{k-1}}^4}$ in the upper bound \eqref{e:kl-ineq} depends on the choice of discretization points.  In particular,
\begin{itemize}
\item If we choose uniform discretization $h_k = c$, the dependence on $\frac{1}{\delta}$ becomes linear.
\item \cite{SGMSDEsong} considers variance function $g(t) = \sqrt{t}$ with uniform discreitzation. This is equivalent to using constant variance function with quadratic discretization points $t_k = (\delta + kh)^2 $ for appropriate $h$. This choice of discretization points induces a linear step size and our Theorem results in a square-root dependence on $\frac{1}{\delta}$.
\item In Theorem \ref{general}, by using exponentially decaying (and then constant) step size, we reduce this error to a logarithmic dependence. Indeed, the term $\Pi$ achieves its minimum (up to a constant) under our choice of discretization points.
\end{itemize}
 See Appendix \ref{discuss-discretization} for details. 
Although under our assumptions, the theory suggests that exponentially decreasing step sizes are optimal, other issues may arise in practice. 
We leave an experimental comparison of different $g$'s or step sizes to future work.

\paragraph*{Wasserstein+KL Guarantee.}
Notice that when $\delta$ is small, $p_{\delta}$ is only a small perturbation (in Wasserstein distance) of the data distribution $P$. Then stopping the algorithm at appropriate $\de$ results in a distribution that is close in KL divergence to a distribution that is close to $P$ in Wasserstein distance, and we obtain the following.
\begin{corollary}\label{cor:main}
Suppose that Assumptions \ref{as1} and \ref{as4} hold for data distribution $P$. 
Then using the exponential integrator scheme with exponentially decreasing step size, to reach a distribution $Q$ 
such that $W_2^2(P,M_\sharp p_\delta)\le \ew^2\le \fc d2$ and $\KL(M_\sharp p_\delta \|Q)\le \epsilon_{\KL}^2\le \fc{d+M_2}2$ requires 
\[
N = {\Theta} \left(\frac{d^2  \log^2\pf{(d+M_2)d}{\ep_{\KL}^2\ew^2}}{\epsilon_\KL^2} \right) 
\]
steps and Assumption \ref{as1} to hold with 
\[
\ep_0^2 \le \frac{\ep_{\KL}^2}{K\log^2\pf{d+M_2}{\ep_{\KL}^2}}
\]
for an appropriate absolute constant $K$.
Here, $M(x) = \exp(\fc\de2)x,\,Q=M_\sharp \hat q_{T-\delta}$.
\end{corollary}
\begin{remark}
Corollary~\ref{cor:main} implies an upper bound for the bounded Lipschitz metric between the data distribution and $\hat{p}_{t_0}$(as mentioned in \cite{SamplingEasy}):
\begin{align*}
               \sup \left\{\EE_\mu f - \EE_\nu f:\left|\right. f:\RR^d \to [-1,1]\,\text{is 1-Lipschitz}\right\}.
\end{align*}
\end{remark}
Note our improved dependencies compared with \cite[Corollary 3]{SamplingEasy} and \cite[Theorem 2.1]{convergencescore2}.


While the smoothness assumption is relaxed, our analysis induces an additional $d$-factor in place of the Lipschitz constant of $\nabla \log p_t$ compared to Theorem \ref{smoothbound}. This 
 $d$-factor comes from the high-probability bound for the Hessian matrix (see Lemma \ref{hessian-moment}). However, \cite[Theorem 5]{SamplingEasy} suggests that the lower bound of the discretization error 
 scales linearly on $d$. 
 We leave open the problem of closing the gap between the dimension dependence in the upper and lower bounds. 
 
\paragraph*{Pure Wasserstein Guarantee.}
We can also obtain a pure Wasserstein guarantee by following \cite[Theorem 2.2]{convergence-score}. For this, we need to include an extra truncation step 
on the algorithm output, i.e., for some choice of $R$, replacing any sample $\hat{y}_{T-\delta} \sim \hat{q}_{T-\delta} $ falling outside $B_R(0)$ by 0. In addition, we need to assume some concentration for $P$, so that samples from $P$ lie in $B_R(0)$ with high probability.
\begin{corollary}\label{cor:W2}
Consider the distribution $\hat{q}_{T-\delta}^{\mathrm{trunc}}$ obtained by exponential integrator scheme with exponentially decreasing step size and the truncation step. Suppose that Assumptions \ref{as1} and \ref{as4} hold with $\ep_0 = O\pf{\ew^2}{R^2}$, and that $R\ge M_2,\delta,T,N$ satisfy
\begin{align} \label{parameter}
\begin{aligned} 
& \delta = \Theta\left(\frac{\ew^2}{d} \right),\quad R^2 \PP(\|x_{\delta}\| \geq R) = O(\ew^2),\\  & T =\Theta\left( \log\left(\frac{R^4(M_2+d)}{\ew^4}\right)\right),\\  & N = \Theta \left(\frac{d^2R^4\left(\log\frac{1}{\de}+T \right)^2}{\ew^4} \right).
\end{aligned}  
\end{align}
Then 
the resulting truncated and scaled distribution $M_\sharp \hat{q}_{T-\delta}^{\mathrm{trunc}}$ satisfies
$
 W_2^2(P, M_\sharp\hat{q}_{T-\delta}^{\mathrm{trunc}}) = \tilde{O}(\ew^2) $. (Here, $M$ is as in Corollary~\ref{cor:main}.)
\end{corollary}
\begin{remark}
Note that the appropriate $R$ in \eqref{parameter} exists under mild tail conditions on the data distribution $P$. For example: 
\begin{itemize}
\item If there exists a constant $\eta >0$ such that $\EE_P\|x\|^{2+\eta} =O(\mathrm{poly}(d))$, $R$ depends polynomially on $\frac{1}{\ew}$ and $d$ and thus we obtain a polynomial complexity guarantee.
\item When the data distribution $P$ is $K$ sub-exponential, $R$ has a logarithmic dependence on $\frac{1}{\ew}$ and \eqref{parameter} induces $N = \tilde{O}\left(\frac{d^2K^4}{\ew^4}\right)$.
\end{itemize}

\end{remark}

\subsection{Result for Smooth Data Distributions} 
We further provide convergence analysis for smooth $p_0$ without using early stopping. As mentioned in Subsection \ref{sub-general}, the early stopping technique is employed to bound the discretization error in the low-noise regime. We can alternatively bound this error by using the smoothness condition on $p_0$:
\begin{assumption} \label{as5} 
The data distribution admits a density $p_0 \in C^2(\RR^d)$ and $\nabla \log p_0$ is $L$-Lipschitz. 
 \end{assumption}
We bound the discretization error in two different time regimes: Choosing an appropriate constant $\delta_0 > 0 $, when $t > \delta_0$, we use a high-probability Hessian bound and a change of measure argument similar to the analysis in the early stopping setting; for $t < \delta_0$, we alternatively derive a Lipschitz constant bound for $\nabla \log p_t$ (stated in Lemma \ref{lownoise}) based on Assumption \ref{as5}.
\begin{theorem} \label{truncating}
There is a universal constant $K$ such that the following holds. 
Under Assumptions \ref{as1}, \ref{as4}, and \ref{as5} hold, by using the exponentially decreasing (then constant) step size $h_k = c\min\{\max\{t_k,\rc L\},1\}$, $c = \fc{\log L + T}{N}\le \rc {Kd}$, the sampling dynamic \eqref{semi-discrete} results in a distribution $\hat q_T$ such that 
\[
\mathrm{KL}(p_0 \| \hat{q}_T) \lesssim  (M_2+d)\exp(-T) + T\epsilon_{0}^2 + \frac{d^2(\log L + T)^2}{N}.\]
Choosing $T=\log \left(\frac{M_2+d}{\ep_0^2}\right)$ and $N=\Theta\left(\frac{d^2(T+\log L)^2}{\ep_0^2}\right)$ makes this $\widetilde{O}\left(\ep_0^2\right)$.

 In addition, for Euler-Maruyama scheme \eqref{EM-discrete}, the same bounds hold with an additional $M_2\sum_{k=1}^N h_k^3 $ term. 
\end{theorem}
Comparing to Theorem \ref{smoothbound}, this result only depends on the Lipschitz constant of $\nabla \log p_0$ rather than the uniform Lipschitz constant bound for $\nabla \log p_t,\, 0\leq t \leq T$. We also ease the dependency on $L$ from $L^2$ to $\log^2 L$ for optimal choice of variance function or step size, 
so the requirement on the smoothness of the data distribution is significantly relaxed: even if the Lipschitz constant $L$ scales exponentially on $d$, we can still obtain a polynomial complexity guarantee.  Note that we do pay an extra $d$ factor compared to Theorem \ref{smoothbound}.
\section{Proof sketches}
\label{s:sketch}
We sketch the proofs of the main theorems using the exponential integrator discretization, and give complete proofs in Appendices~\ref{overview} and~\ref{s:proofs}. We first consider the smooth setting, and then describe the modifications for the non-smooth case. Our main technical novelty lies in the arguments for the non-smooth setting, we also streamline the arguments in the smooth setting and use an interpolation rather than Girsanov approach that gives KL divergence bounds.

\subsection{Smooth setting (Theorem~\ref{smoothbound})}


\paragraph{First term.}
The first source of error arises from the mismatch between the distribution of the forward process $p_T$ at time $T$, and our Gaussian initialization for the reverse process, $\hat q_0=\gamma_d$. We can separate out this term using the chain rule for KL divergence:
\begin{align*}
\KL(p_0\|\hat  q_T) &\le \KL(p_T \|\hat  q_0) + \EE_{p_T(a)}\KL (p_{0|T}(\cdot|a)\| \hat  q_{T|0} (\cdot|a)).
\end{align*}
The first term can be bounded using exponential mixing of the forward (Ornstein-Uhlenbeck) process towards the standard Gaussian. In conjunction with the fact that after constant time, the KL-divergence is bounded by $O(d+M_2)$, we obtain (Lemma~\ref{converge-forward})
\begin{align*}
    \KL(p_T \|\hat  q_0)&\lesssim (d+M_2) e^{-T}.
\end{align*}
Note this estimate does not depend on the initial distance $\mathrm{KL}(p_0\|\gamma_d) $ as in \cite{SamplingEasy}.

The remaining term can be written as a sum, again using the chain rule for KL divergence, by comparing the continuous process with the estimated, discrete process through a chain of intermediate processes where we run the continuous process until time $t_k$. We can interpolate the discrete processes to realize them as SDE's. If Novikov's conditions are satisfied, Girsanov's Theorem then applies to bound the KL divergence in terms of the squared difference of the drift terms between the processes.
$$
\begin{aligned}
&\EE_{p_T(a)} \KL(p_{T|0}(\cdot|a)\| \hat  q_{T|0} (\cdot|a))\\
&=\sum_{k=1}^N \EE_{p_{t_k}(a)} \KL(p_{t_{k-1}|t_k}(\cdot|a) \| \hat  q_{T-t_{k-1}|T-t_k}(\cdot |a))\\
&\le \sum_{k=1}^N \rc 2\int_{t_{k-1}}^{t_{k}} \EE_{x_t\sim p_t} \ve{s(x_{t_{k}}, t_{k}) - \nb \log p_t(x_t)}^2\,dt\\
&\le \underbrace{\sum_{k=1}^N \int_{t_{k-1}}^{t_{k}}
\EE_{x_t\sim p_t} \ve{s(x_{t_{k}}, t_{k}) - \nb \log p_{t_k}(x_{t_k})}^2}{(2)} \\
&\qquad + \underbrace{\sum_{k=1}^N\EE \ve{ \nb \log p_{t_k}(x_{t_k}) - \nb \log p_t(x_t)}^2\,dt}{(3)}.
\end{aligned}
$$
In the last step we use the triangle inequality. However, in general Novikov's condition may not be satisfied; \cite{SamplingEasy} circumvent this using an involved truncation argument which only results in a TV bound and relies on the trajectory-smooth condition (Assumption \ref{as2}). We instead use a differential inequality argument which gives the same conclusion (Lemma \ref{diff}, \ref{tech}, Proposition \ref{Girsanov}) and is applicable to the non-smooth setting; this step requires significant technical work (Appendix~\ref{C}). 

\paragraph{Second term.}
Term (2) is exactly the score estimation error, and by Assumption~\ref{as1}, it is bounded by $T\ep_0^2$. 

\paragraph{Third term.}
Term (3) is the discretization error. This discretization error bound is non-trivial since in classical numerical analysis theory, the discretization error often depends exponentially on the time $T$ due to the use of Gronwall's inequality. Our analysis our will rely on the special structure of the Ornstein-Uhlenbeck process.
We note that (3) involves both a ``time" and ``space" discretization error (as both the time and space arguments are different). We show in Lemma~\ref{time-discrete} that this can be bounded purely in terms of the space discretization error (which streamlines the argument of~\cite{SamplingEasy})
\begin{multline*}
\EE \ve{\nb \log p_{s} (x_{s}) - \nb \log p_t(x_t)}^2\lesssim (s-t)^2\cdot \\
\EE\ve{\nb \log p_t(x_t)}^2
+
\EE \ve{\nb \log p_t(x_t) - \nb \log p_t(\alpha_{s,t}^{-1} x_s)}^2.
\end{multline*}
The explicit form of the OU process tells us that $\alpha_{s,t}^{-1} x_s = x_t+z$, where $z$ is a Gaussian of variance $O(s-t)$. Therefore, the second term (which dominates) can be bounded as a Lipschitz constant times the second moment of a Gaussian:
\begin{multline}
\EE \ve{\nb \log p_t(x_t) - \nb \log p_t(\al_{s,t}^{-1} x_s)}^2 
\lesssim L^2\EE \ve{z}^2 \\
\lesssim dL^2(s-t).
\label{e:space-disc}
\end{multline}
Note that we crucially use the Lipschitzness of the score in this step. Plugging this bound into the sum (3) gives the final error term.

\subsection{Non-smooth setting (Theorem~\ref{general})}
Comparing Theorem~\ref{smoothbound} (smooth setting) and Theorem~\ref{general} (non-smooth setting), we note that the discretization error changes from $\fc{T^2L^2d}{N}$ to $\fc{\pa{\log\prc\de + T}^2d}N$; the intuition is that $L$ is ``effectively" bounded by $\sqrt d$.
Previously, \cite{SamplingEasy} assume that $P$ is supported on a ball of radius $R$ to derive a global Lipschitzness bound $\ve{\nb^2 \log p_t}=O\pf{R^2}{t^2}$ to plug into the smooth theorem. 

Our main insight is that (1) because we are averaging the error over $p_t$, it suffices to have a high-probability rather than uniform bound on the the Hessian, and (2) such bounds are obtainable from the smoothing properties of the forward process. In fact, to bound~\eqref{e:space-disc}, we only need Lipschitzness in a \emph{random} direction, and hence a Frobenius norm bound is sufficient (Lemma \ref{hessian-moment}):
\begin{align}\label{e:hess-hiprob}
\ve{\ve{\nb^2 \log p_t(x)}_F}_{\psi_1} \lesssim \fc{d}{\min\{t,1\}}.
\end{align}
(This is the weaker analogue of an operator norm bound of $O(\sqrt d)$, which was suggested from the $L=O(\sqrt d)$ analogy.) This incurs significant savings over a uniform bound, and in particular does not depend on boundedness or tails of $P$. 
We prove this by giving a Bayesian interpretation of the Hessian as the posterior variance of the noise in the score matching objective. As a purely mathematical statement about smoothing of the OU process, this result may be of independent interest.

Finally, to use~\eqref{e:hess-hiprob} in~\eqref{e:space-disc}, we actually need to bound the Hessian not just at $x_t$ but along the path (in direction $z$) joining $x_t$ and $\alpha_{s,t}^{-1} x_s$: for this we need a change-of-measure argument (Lemma \ref{noLip}) which says that the distributions of $(x_t,z)$ and $(x_t + az, z)$ are close in $\chi^2$-divergence, for $0\le a \le 1$. Finally, although the bound~\eqref{e:hess-hiprob} blows up as $t\to 0$, by choosing an exponentially decreasing step size and stopping at time $\delta$, we only incur a $\log\prc{\de}$ dependence, similarly to the analysis of the score estimation error (Remark~\ref{r:avg}).

\subsection{Smooth $p_0$ (Theorem~\ref{truncating})}

If we only assume $\nabla \log p_0$ is $L$-Lipschitz, we can still derive Lipschitzness of $\nabla \log p_0$ for small time $t\le \rc L$ (Lemma~\ref{lownoise}). For large $t\ge L$, the argument in the non-smooth case applies (and gives a bound of $O(dL)$ in~\eqref{e:hess-hiprob}). Thus, we take exponentially decreasing step size until $t=1/L$, and then constant step size, and combine the analyses of Theorems~\ref{smoothbound} and~\ref{general} to obtain Theorem~\ref{truncating}.

\section{Conclusion}
In this paper, we analyzed the theoretical properties of SGM in various regimes. We extended existing result to the most general setting and provided refined guarantees. The current analysis provides guarantees for SGM in the framework that an $L^2$-accurate score estimator is available. This implies the training objective in denoising score matching is suitable for learning a generative model and partially explains why SGM is empirically successful at modeling very complex distributions, like multi-mode distributions or distributions with weak smoothness condition. 

We obtain guarantees for arbitrary data distributions without smoothness assumptions, by exploiting (high-probability) smoothing properties of the forward process. Besides closing the factor-$d$ gap between our upper bound and the (suggested) lower bound, it would be interesting to carry out this kind of analysis for other choices of the forward/backward processes, such as critically damped Langevin Diffusion~\cite{dockhorn2021score}, to see if improved guarantees are available. (\cite{SamplingEasy} show that no improvement is available only in the setting of a uniform bound on the Lipschitz constant of the score.)

Another future direction is to explore theories beyond the framework that an $L^2$-accurate score estimator is available and understand the learning of a score estimator, including the approximability, sample complexity, and the training dynamics of denoising score matching. This is related to the most challenging problems in deep learning theory; advances in deep learning theory may provide some new insight into SGM.

\bibliographystyle{amsalpha}
\bibliography{ref}
\newpage
\appendix

\section{Denoising Score Matching} \label{DSM}
For $0\leq t\leq T$, the goal of score matching for $p_t$ is to minimize
\begin{align*}
    \min_{\theta}   \EE_{p_t}\|s_\theta(t,x) - \nabla \log p_t(x) \|^2.
\end{align*}
Since the score function $\nabla \log p_t$ is not available, we alternatively consider a denoising score matching objective \cite{Vincent2011ACB}, which is derived from integrating by parts
\begin{align*}
    & \EE_{p_t} \|s_\theta(x,t) - \nabla \log p_t(x) \|^2 \\ &= \EE_{p_t} \|s_\theta(x,t)\|^2 + \EE_{p_t} \|\nabla \log p_t \|^2 - 2\EE_{p_t}\langle s_\theta(x,t), \nabla \log p_t(x) \rangle  \\
    & =  \EE_{p_t} \|s_\theta(x,t)\|^2 + \EE_{p_t} \|\nabla \log p_t(x) \|^2 - 2\EE_{p_t}  \nabla \cdot s_\theta(x,t) \\
    & = \EE_{p_t} \|s_\theta(x,t)\|^2 + \EE_{p_t} \|\nabla \log p_t(x) \|^2 - 2\EE_{p_0(x_0)}\EE_{p_{t|0}(x_t|x_0)}  \nabla \cdot s_\theta(x_t,t)  \\
    & =  \EE_{p_t} \|s_\theta(x,t)\|^2 + \EE_{p_t} \|\nabla \log p_t(x) \|^2 - 2\EE_{p_0(x_0)}\EE_{p_{t|0}(x_t|x_0)}  \langle  \nabla \log p_{t|0}(x_t|x_0), s_\theta(x_t,t) \rangle \\
    & = \EE_{p_t} \|s_\theta(x,t)\|^2 + \EE_{p_t} \|\nabla \log p_t(x) \|^2 - 2\EE_{p_0(x_0)}\EE_{p_{t|0}(x_t|x_0)}  \left\langle   \frac{x_t-\alpha_t x_0}{\sigma_t^2}, s_\theta(x_t,t) \right\rangle  \\
    & =  \EE \left\|s_\theta(x_t,t) -\frac{x_t-\alpha_t x_0}{\sigma_t^2}  \right\|^2 +  \EE_{p_t} \|\nabla \log p_t(x) \|^2  - \frac{d}{\sigma_t^2} \\
    & =  \EE \left\|s_\theta(x_t,t) -\frac{x_t-\alpha_t x_0}{\sigma_t^2}  \right\|^2 + C,
\end{align*}
where $p_{t|0}$ is the conditional distribution of $x_t$ given $x_0$, and $C$ is a constant independent of $\theta$. 

Noticing that $\mathbb{E}\left\|\frac{x_t-\alpha_t x_0}{\sigma_t^2}\right\|^2=\frac{1}{\sigma_t^2}$, it is natural to expect the error to scale as
$$\EE_{p_{t_k}} \|\nabla \log p_{t_k}(x)- s(x,t_k) \|^2 \lesssim \frac{\epsilon^2}{\sigma_{t_k}^2}.$$ 
In this case, by noting that $  \sigma_{t_k}^2 \asymp \min\{1,t_k\}$, we have 
$$
\EE_{p_{t_k}} \left\|\nabla \log p_{t_k}(x)-s\left(x, t_k\right)\right\|^2 \lesssim \frac{\ep^2}{\min \left\{t_k, 1\right\}},
$$
then \eqref{score-error} is satisfied with a log factor:
$$
\frac{1}{T} \sum_{k=1}^T h_k\EE_{p_{t_k}}\left\|\nabla \log p_{t_k}(x)-s\left(x, t_k\right)\right\|^2 \lesssim \frac{1}{T} \int_{t_1}^T \frac{\ep^2}{t \wedge 1} d t \lesssim \ep^2 \log \left(\frac{1}{t_1}\right)
$$

 \section{Discussion on Choices of Discretization Points} \label{discuss-discretization}
 In this section, we consider the scaling of the term $\Pi = \sum_{k=1}^N \frac{h_k^2}{\sigma_{t_{k-1}}^4}$ in \eqref{e:kl-ineq} under different choices of discreitzation points.
 \paragraph*{The Constant Step Size}
 For uniform discretization(inducing constant step size) $t_k = \delta + kh,\, h = \frac{T-\delta}{N}$, we have
 \begin{align*}
 \sum_{k=1}^N \frac{h_k^2}{\sigma_{t_{k-1}}^4} & \asymp \sum_{t_k \leq 1} \frac{h_k^2}{t_{k-1}^2} + \sum_{t_k > 1} h_k^2 \\
& \asymp h\int_\delta^1 \frac{1}{t^2} \d t + \frac{T^2}{N} \\
 &  \asymp  \frac{T/\delta + T^2}{N}.
 \end{align*}
 Thus the upper bound for discretization error has a linear dependence on $\frac{1}{\delta}$.
 \paragraph*{The Linear Step Size}
For quadratic discretization points(inducing linear step size) $t_k = (\delta + kh)^2,\, h = \frac{\sqrt{T} - \delta}{N} $, by noting that $\frac{h_k}{h} \asymp \sqrt{t_k} $, we have
 \begin{align*}
     \sum_{k=1}^N \frac{h_k^2}{\sigma_{t_{k-1}}^4} &  \asymp \sum_{t_k \leq 1} \frac{h_k^2}{t_{k-1}^2} + \sum_{t_k > 1} h_k^2  \\
     & \asymp h\sum_{t_k \leq 1} \frac{h_k}{t_k^{3/2}} + h \sum_{t_k > 1} \sqrt{t_k} h_k  \\
     & \asymp h\int_\delta^1 \frac{1}{t_k}\d t + h \int_1^T \sqrt{t} \d t \\
     & \asymp \frac{1}{N}\left(\sqrt{\frac{T}{\delta}} + T^2  \right).
     \end{align*}
 \paragraph*{Optimality of Exponential Decaying Step Size}
Now we will show that the discretization points used in Theorem~\ref{general} minimizes the term $\Pi$ (up to a constant). Indeed, note that
 $$\Pi \asymp \Pi_1 + \Pi_2,\quad \Pi_1:=\sum_{t_k \leq 1}\frac{(t_k - t_{k-1})^2 }{t_k^2},\, \Pi_2 := \sum_{t_k> 1} (t_k - t_{k-1})^2.$$
 
 For the term $\Pi_1$, let $z_k = \log \frac{t_k}{t_{k-1}} > 0$, we have $\Pi = \sum_{k=1}^n (e^{z_k} - 1)^2$.
 Note that $z \mapsto (e^z -1)^2$ is convex for $z>0$. By Jensen's inequality, when the summation of $z_k$'s are fixed, the minimum of $\Pi_1$ is reached when $z_k$'s are identical. Equivalently, $h_k = ct_k$ for $t_k \leq 1$. For the term $\Pi_2$, we have $ \Pi_2 = \sum_{t_k > 1}h_k^2$. Similarly, since $ h \mapsto h^2$ is convex for $h>0$, the minimum of $\Pi_2$ is reached when $h_k$'s are identical.
\section{Main Proof Ingredients}
\label{overview}
The key idea of the proof is motivated by the Girsanov change of measure framework used in \cite{SamplingEasy}. However, in order to avoid the technical challenge of altering the process to satisfy Novikov's condition, we use a differential inequality-based argument instead.
\begin{lemma} \label{diff}
Consider the following two It\^o processes 
\begin{align*}
    \d X_t & = F_1(X_t,t) \d t + g(t) \d w_t,\, &  X_0 &= a, \\
    \d Y_t & = F_2(Y_t,t) \d t + g(t) \d w_t,\, &  Y_0 &= a, 
\end{align*}
where $F_1,F_2,g$ are continuous functions and may depend on $a$. We assume the uniqueness and regularity condition:
\begin{itemize}
    \item The two SDEs have unique solutions.
    \item $X_t,Y_t$ admit densities $ p_t,q_t \in C^2(\RR^d)$ for $t>0$. 
\end{itemize}  
Define the relative Fisher information between $p_t$ and $q_t$ by
\begin{align*}
    J(p_t \| q_t) = \int p_t(x) \left\|\nabla \log \frac{p_t(x)}{q_t(x)} \right\|^2 \d x.
\end{align*}
Then for any $t>0$, the evolution of $\KL(p_t \| q_t)$ is given by
\begin{align*}
\dt \KL(p_t \| q_t) = -g(t)^2 J(p_t \| q_t) + \EE \left[\left\langle F_1(X_t,t) - F_2(X_t,t) , \nabla \log \frac{p_t(X_t)}{q_t(X_t)} \right\rangle \right].
\end{align*}
\end{lemma}

\begin{remark}
While we have written the same Brownian motion for $X$ and $Y$, as we only care about distributions, the Brownian motions can be chosen independent with each other.  
\end{remark}

We will apply Lemma \ref{diff} on $(\tilde{x}_t)_{0\leq t \leq T-\delta}$ and $(\hat{y}_t)_{0\leq t \leq T-\delta}$ to show the convergence in KL divergence. The following lemma collects some technical properties of the two processes. The proof of both lemmas is deferred to  Appendix \ref{C}. 

\begin{lemma} \label{tech}
  For $0\leq k \leq N-1$, consider the reverse SDE starting from $\tilde{x}_{t_k'} = a$
\begin{align} \label{backward-local}
    \d \tilde{x}_t = \left[\frac{1}{2} \tilde{x}_t +  \nabla \log \tilde{p}_{t}(\tilde{x}_t) \right]\d t +  \d {w}_t,\quad  \tilde{x}_{t_k'} = a
\end{align}
and its discrete approximation:   
\begin{align}  \label{discrete-local}
    \d \hat{y}_t = \left[\frac{1}{2} \hat{y}_t + s(a, T-t_k')\right] \d t + \d w_t,\quad \hat{y}_{t_k'} = a\, 
    \end{align}
for time $t \in (t_k', t_{k+1}']$. Let $\tilde{p}_{t|t_k'}$ be the density
of $\tilde{x}_t$ given $\tilde{x}_{t_k'} $ and $\hat{q}_{t|t_k'}$ be density of $\hat{y}_t $ given $\hat{y}_{t_k'}$. Then we have
\begin{enumerate}
    \item \label{1.} For any $a \in \RR^d $, the two processes satisfy the uniqueness and regularity condition stated in Lemma \ref{diff}, that is, \eqref{backward-local} and \eqref{discrete-local} have unique solution and $\tilde{p}_{t|t_k'}(\cdot|a),\hat{q}_{t|t_k'}(\cdot|a) \in C^2(\RR^d)$ for $t>t_k'$.
    \item \label{2.} For a.e. $a\in \RR^d$ (with respect to the Lebesgue measure), we have
    $$ \lim_{t \to {t_k'}+} \KL(\tilde{p}_{t|t_k'}(\cdot | a) \| \hat{q}_{t|t_k'}(\cdot |a)) = 0.$$ 
\end{enumerate}
In addition, the above results also hold if we replace $\hat{y}_t$ with that corresponding to the Euler-Maruyama scheme:
\begin{align*} 
    \d \hat{y}_t = \left[\frac{1}{2}g(T-t)^2 a + g(T-t)^2 s_\theta(a, T-t_k')\right] \d t +\d w_t,\, \hat{y}_{t_k'} = a.
    \end{align*}
\end{lemma}
\begin{proposition}\label{Girsanov}
Under Assumption \ref{as1}, we have
\begin{itemize} 
    \item The exponential integrator scheme \eqref{semi-discrete} satisfies
\begin{align*}
\mathrm{KL}(p_\delta\| \hat{q}_{T-\delta}) \lesssim \mathrm{KL}(p_T\| \gamma_d) + T\epsilon_0^2 + \sum_{k=1}^N \int_{t_{k-1}}^{t_k} \EE \|\nabla \log p_t(x_t) - \nabla \log p_{t_k}(x_{t_k}) \|^2 \d t.
\end{align*}
\item The Euler-Maruyama scheme \eqref{EM-discrete} satisfies
\begin{align*}
\mathrm{KL}(p_\delta \| \hat{q}_{T-\delta}) &\lesssim  \mathrm{KL}(p_T\| \gamma_d) + T\epsilon_0^2  \\
& \quad + \sum_{k=1}^N \int_{t_{k-1}}^{t_k} \left(\EE \|\nabla \log p_t(x_t) - \nabla \log p_{t_k}(x_{t_k}) \|^2 + \EE \|x_t - x_{t_k} \|^2 \right)\d t.
\end{align*}
\end{itemize}
\end{proposition}
 \begin{proof}
Let us consider first the exponential integrator. For $t_k' < t \leq t_{k+1}'$, let $\tilde{p}_{t|t_k'}$ be the distribution of $\tilde{x}_t$ given $\tilde{x}_{t_k'}$ and $\hat{q}_{t|t_k'}$ be the distribution of $\hat{y}_t$ given $\tilde{y}_{t_k'}$. From Lemma \ref{tech}(\ref{1.}) the uniqueness and regularity condition in Lemma \ref{diff} hold for \eqref{backward-local} and \eqref{discrete-local}. Thus for any $a\in \RR^d$ and $t > t_k'$ we have
 \begin{align}
      \frac{\d}{\d t} \KL(\tilde{p}_{t|t_k'}(\cdot|a) \|\hat{q}_{t|t_k'}(\cdot|a) ) & = - {\frac12}  \EE_{\tilde{p}_{t|t_k'}(y|a)} \left\|\nabla \log \frac{\tilde{p}_{t|t_k'}(y|a)}{\hat{q}_{t|t_k'}(y|a)} \right\|^2   \notag \\ &\quad + \EE_{\tilde{p}_{t|t_k'}(y|a)} \left[\left\langle ( \nabla \log \tilde{p}_t(y) - s(a, t_{N-k}) ),  \nabla \log \frac{\tilde{p}_{t|t_k'}(y|a)}{\hat{q}_{t|t_k'}(y|a)} \right\rangle \right] \notag \\
     & \leq \frac{1}{2} \EE_{\tilde{p}_{t|t_k'}(y|a)} \| s(a, t_{N-k}) - \nabla \log \tilde{p}_t(y)\|^2, \label{diff-inequality}
 \end{align}
 where we use the fact that $\langle v,w\rangle \le \frac12\|v\|^2 + \frac12 \|w\|^2$. By Lemma \ref{tech}(\ref{2.}), for a.e.{} $a \in \RR^d$ we have
 $$ \lim_{t \to {t_k'}+} \KL(\tilde{p}_{t|t_k'}(\cdot | a) \| \hat{q}_{t|t_k'}(\cdot |a)) = 0, $$
and hence
 \begin{align*}
 \KL(\tilde{p}_{t_{k+1}'|t_k'}(\cdot|a) \|\hat{q}_{t_{k+1}'|t_k'}(\cdot|a) ) \leq \frac{1}{2}\int_{t_k'}^{t_{k+1}'}\EE_{\tilde{p}_{t|t_k'}(y|a)} \| s(a, t_{N-k}) - \nabla \log \tilde{p}_t(y)\|^2 \d t.
 \end{align*}
 Since $\tilde{p}_{t_k'} $ is absolutely continuous w.r.t. the Lebesgue measure,  integrating on the both sides w.r.t.{} $\tilde{p}_{t_k'}$ yields
 \begin{align*}
\EE_{\tilde{p}_{t_k'}(a)}  \KL(\tilde{p}_{t_{k+1}'|t_k'}(\cdot|a) \|\hat{q}_{t_{k+1}'|t_k'}(\cdot|a) ) \leq \frac{1}{2}\int_{t_k'}^{t_{k+1}'}  \EE \| s(\tilde{x}_{t_k'}, t_{N-k}) - \nabla \log \tilde{p}_t(\tilde{x}_t)\|^2  \d t.
 \end{align*}
 For $0\leq  k \leq N-1$, we use the chain rule of KL divergence to obtain
 \begin{align*}
     \KL(\tilde{p}_{t_{k+1}'} \| \hat{q}_{t_{k+1}'} ) & \leq \EE_{\tilde{p}_{t_k'}(a)}  \KL(\tilde{p}_{t_{k+1}'|t_k'}(\cdot|a) \|\hat{q}_{t_{k+1}'|t_k'}(\cdot|a) ) +   \KL(\tilde{p}_{t_k'} \| \hat{q}_{t_k'} ) \\
     &  \leq \KL(\tilde{p}_{t_k'} \| \hat{q}_{t_k'} ) + \frac{1}{2} \int_{t_k'}^{t_{k+1}'}  \EE \| s(\tilde{x}_{t_k'}, T-t_k') - \nabla \log \tilde{p}_t(\tilde{x}_t)\|^2  \d t.
 \end{align*}
Summing over $k=0,1,\ldots,N-1$ and using $p_t = \tilde{p}_{T-t}$, we obtain
 \begin{align*}
     \KL(p_\delta \| \hat{q}_{T-\delta}) & \leq \KL(p_T \| \gamma_d ) + \frac{1}{2} \sum_{k=0}^{N-1}\int_{t_k'}^{t_{k+1}'} \EE \| s(\tilde{x}_{t_k'}, T-t_k') - \nabla \log \tilde{p}_t(\tilde{x}_t)\|^2 \d t \\
     & \leq  \KL(p_T \| \gamma_d ) + \frac{1}{2}\sum_{k=1}^N \int_{t_{k-1}}^{t_k}  \|s(x_{t_k}, t_k) - \nabla \log p_t(x_t) \|^2 \d t \\
     & \leq \KL(p_T \| \gamma_d )  + \sum_{k=1}^N \int_{t_{k-1}}^{t_k}  \| s(x_{t_k}, t_k) - \nabla \log p_{t_k}(x_{t_k})  \|^2 \d t \\ & \quad + \sum_{k=1}^N \int_{t_{k-1}}^{t_k}  \|\nabla \log p_{t_k}(x_{t_k}) - \nabla \log p_t(x_t) \|^2 \d t 
     \\ & \leq \KL(p_T \| \gamma_d ) + T\epsilon_0^2 + \sum_{k=1}^N \int_{t_{k-1}}^{t_k} \|\nabla \log p_{t_k}(x_{t_k}) - \nabla \log p_t(x_t) \|^2 \d t.
 \end{align*}
This completes the proof for the exponential integrator scheme.
 The proof for the Euler-Maruyama scheme is similar; the only difference is the differential inequality becomes
 \begin{align*}
       \frac{\d}{\d t}\EE_{\tilde{p}_{t_k'}} \KL(\tilde{p}_{t|t_k'}(\cdot|x) \|\hat{q}_{t|t_k'}(\cdot|x) ) \leq \frac{1}{2} \EE \left\|   \nabla \log \tilde{p}_t(\tilde{x}_t) - s(\tilde{x}_{t_k'}, t_{N-k}) + \frac{1}{2}(\tilde{x}_t - \tilde{x}_{t_k'}) \right\|^2
 \end{align*}
and we can obtain the result in an analogous way. 
\end{proof}
The three terms in the upper bound of Proposition \ref{Girsanov} match the claim in Theorem \ref{smoothbound}. The first term is controlled by the exponential convergence of the forward process, which is given in the following lemma.
\begin{lemma}\label{converge-forward}
Under Assumption \ref{as4}, for $T>1$, we have
\begin{align*}
    \mathrm{KL}(p_T \| \gamma_d) \leq (d+M_2) e^{-T}.
\end{align*}
\end{lemma}
\begin{proof}
Notice that $x \mapsto x\log x $ is a convex function for $x>0$. Let $p_{t|0}$ be the conditional density of $x_t$ given $x_0$. For any $t>0$, we can use Jensen's inequality to bound the entropy of $p_t$:
\begin{align*}
 \int_{\RR^d} p_t(x) \log p_t(x)\d x & = \int_{\RR^d} \left[\left(\int_{\RR^d}  p_{t|0}(x|y) \d P(y)\right) \log \left( \int_{\RR^d} p_{t|0}(x|y) \d P(y) \right)\right] \d x \\
 & \leq \int_{\RR^d}\left[\int_{\RR^d} p_{t|0}(x|y)\log p_{t|0}(x|y) \d P(y)\right] \d x \\
 & = \int_{\RR^d} \left(\int_{\RR^d} p_{t|0}(x|y)\log p_{t|0}(x|y) \d x\right)\d P(y).
\end{align*}
Since $x_t| x_0 = y \sim \cN(\alpha_t x_0,  \sigma_t^2 I_d )$, we have
\begin{align*}
 \int_{\RR^d} p_{t|0}(x|y)\log p_{t|0}(x|y) \d x  = -\frac{d}{2}\log(2\pi \sigma_t^2 ) - \frac{d}{2}.
\end{align*}
Thus 
\begin{align*}
    \int_{\RR^d} p_t(x) \log p_t(x) \d x \leq -\frac{d}{2}\log (2\pi \sigma_t^2) - \frac{d}{2}.
\end{align*}
Therefore,
\begin{align*}
    \mathrm{KL}(p_t \| \gamma_d )&=  \int_{\RR^d} p_t(x) \log p_t(x) \d x + \EE_{p_t}\left[\frac{\|x\|^2}{2} + \frac{d}{2} \log (2\pi) \right] \\
    & \leq \frac{d}{2}\log \sigma_t^{-2} + \frac{1}{2}(M_2-d).
\end{align*}
From the exponential convergence of Langevin dynamics with strongly log-concave stationary distribution (see, e.g., \cite{Vempala2019RapidCO}), we obtain
\begin{align*}
    \KL(p_T \| \gamma_d ) \leq e^{-T+t}\left(\frac{d}{2}\log \sigma_t^{-2} + \frac{1}{2}(M_2-d) \right).
\end{align*}
By choosing $t = \log 2 $, we have
\begin{align*}
       e^{t}\log\left(\frac{1}{\sigma_t^2} \right)  \lesssim 1.
\end{align*}
Thus
\begin{equation*}
      \KL(p_T \| \gamma_d ) \lesssim e^{-T}(d+M_2).  \qedhere
\end{equation*}
\end{proof}

The second term in the upper bound of Proposition~\ref{Girsanov} is exactly the same as the score estimation error defined in Assumption \ref{as1}. So the key challenge is to bound the third term, which is caused by the discretization error.

According to Proposition \ref{Girsanov}, the discretization error of the Euler-Maruyama scheme induces an extra linear term $\EE \|x_t - x_{t_k} \|^2$ compared to the exponential integrator scheme. The following lemma bounds this extra term.
\begin{lemma} \label{linear-discretization}
Suppose that $h_k \leq 1$ for $1\leq k \leq N$. We have 
    $$\EE \|x_t - x_{t_k} \|^2 \lesssim d (t_k-t) + M_2 (t_k-t)^2,\quad  t_{k-1} \leq t \leq t_k, $$
and
\begin{align*}
    \sum_{k=1}^N \int_{t_{k-1}}^{t_k} \EE \|x_t - x_{t_k} \|^2 \d t \lesssim d \sum_{k=1}^N h_k^2 + M_2 \sum_{k=1}^N h_k^3.
\end{align*} 
\end{lemma}
\begin{proof}
From the definition of the forward process \eqref{forward}, we have
     \begin{align} 
 \EE\|{x}_t - {x}_{t_k}\|^2  \notag  & =  \EE \left\|\int_t^{t_k} \frac{1}{2} {x}_{u} \d u - \int_{t}^{t_k'}  \d w_u \right\|^2  \notag \\
         & \lesssim \EE \left\|\int_t^{t_k}  x_u \d u \right\|^2 +  \left\|\int_t^{t_k}  \d w_u \right\|^2  \notag \\
         &  \leq (t_k-t) \left(\int_t^{t_k}  \EE \|x_u \|^2 \d u \right) + d(t_k-t), \label{112}
     \end{align}
     where the last inequality follows from the Cauchy-Schwartz inequality. From the explicit form of the conditional density
\begin{align*} 
        x_{u} | x_0 \sim   \cN\left(e^{-\frac{1}{2}u}x_0, \left( 1-e^{-u} \right)I_d \right)           ,
\end{align*}
the second moment of $x_{u}$ is bounded by $\EE \|x_{u} \|^2 \leq M_2 + d$. Pluging this into \eqref{112}, we arrive at 
\begin{align*}
     \EE \|x_t - x_{t_k}\|^2 \lesssim d(t_k-t) + (d+M_2)(t_k-t)^2.
\end{align*}
Therefore, 
\begin{align*}
\int_{t_{k-1}}^{t_k}  \EE\|x_t - x_{t_k} \|^2 &\lesssim dh_k^2 + (d+M_2)h_{k}^3. 
\end{align*}
Taking summation over $k=1,\ldots,N$, we complete the proof.
\end{proof}

Therefore, we only need to focus on the term $\EE \|\nabla \log p_t(x_t) - \nabla \log p_{t_k}(x_{t_k})\|^2$.
This discretization is taken both in space and time. One observation is that the time-discretization error can be absorbed by the space-discretization error.
\begin{lemma}  \label{time-discrete}
For any $0 \leq t \leq s \leq T$, the forward process \eqref{forward} satisfies
\begin{multline*}
    \EE \|\nabla \log p_t(x_t) - \nabla \log p_s(x_s) \|^2 \leq 4 \EE \|\nabla \log p_t(x_t) - \nabla \log p_t(\alpha_{t,s}^{-1}x_s) \|^2  + 2\EE\| \nabla \log p_t(x_t)\|^2 \left(1-\alpha_{t,s}^{-1}\right)^2.
\end{multline*}
\end{lemma}
\begin{proof}
Since $x_s | x_t \sim \cN \left(\alpha_{t,s} {x}_{t}, (1-\alpha_{t,s}^2) I_d\right)$, from Lemma \ref{computing-score}, we can rewrite $\nabla \log {p}_{s}$ as 
\begin{align*}
   \nabla \log p_s(x)  = \alpha_{t,s}^{-1} \EE_{p_{t|s}(y|x)} \nabla_y \log p_t(y), 
\end{align*}
where $p_{t|s}$ is the conditional density of $x_t$ given $x_s$. 
Thus the time discretization error can be bounded by
\begin{align*}
     \EE \|\nabla \log {p}_{t}(\alpha_{t,s}^{-1}x_s) & - \nabla \log {p}_{s}({x}_{s})  \|^2  = \EE_{p_s} \left\|\alpha_{t,s}^{-1} \EE_{p_{t|s}(y|x_s)}\nabla \log {p}_t(y)  - \nabla \log {p}_t(\alpha_{t,s}^{-1}x_s) \right\|^2 \\
    & \leq \EE \|\alpha_{t,s}^{-1} \nabla \log p_t(x_t) - \nabla \log p_t(\alpha_{t,s}^{-1}x_s) \|^2  \\
    & \leq 2(1-\alpha_{t,s}^{-1})^2 \EE \|\nabla \log p_t(x_t) \|^2 + 2\EE \|\nabla \log p_t(x_t) - \nabla \log p_t(\alpha_{t,s}^{-1}x_s) \|^2.
\end{align*}
Therefore, splitting the error into the space-discretization and the time-discretization error,
\begin{align*}
  &   \EE \|\nabla \log {p}_{t}(x_t) - \nabla \log {p}_{s}(\alpha_{t,s}^{-1}{x}_{s})  \|^2 \\   & \leq 2   \EE \|\nabla \log {p}_{t}(x_t) - \nabla \log {p}_{t}(\alpha_{t,s}^{-1}{x}_{s})  \|^2 +   2 \EE \|\nabla \log {p}_{t}(\alpha_{t,s}^{-1}x_s) - \nabla \log {p}_{s}({x}_{s})  \|^2 \\ 
     & \leq 2(1-\alpha_{t,s}^{-1})^2 \EE \|\nabla \log p_t(x_t) \|^2 + 4\EE \|\nabla \log p_t(x_t) - \nabla \log p_t(\alpha_{t,s}^{-1}x_s) \|^2.
\end{align*}
We complete the proof.
\end{proof}
In Lemma \ref{time-discrete}, $(1-\alpha_{t,s}^{-1})^2 = O((s-t)^2)$ and the term $\EE \|\nabla \log p_t(x_t) \|^2$ can be bounded by Lemma \ref{score-bound}, so the space-discretization error dominates the right hand side. In what follows, we tackle the space-discretization term $\EE \|\nabla \log p_t(x_t) - \nabla \log p_t(\alpha_{t,s}^{-1}x_s) \|^2$ in various regimes. In particular:
\begin{itemize}
    \item If the score functions of the forward process is smooth, i.e., Assumption \ref{as2} holds, the space-discretization error can be directly bounded using the Lipschitz condition on $\nabla \log p_t$.
    \item In the general setting, we choose a early stopping time $t_0$ and bound the space-discretization error for $t>t_0$ by a high-probability bound on the Hessian matrix $\nabla^2 \log p_t$ and a change of measure argument, which are worked out in section~\ref{3.1}.
    \item For smooth $p_0$, we further bound the space-discretization error for small $t$ by providing a Lipschitz constant bound for $\nabla \log p_t$ when $t$ is sufficient small, which is given in section~\ref{3.2}.
\end{itemize}
\subsection{The High-probability Hessian Bound and Change of Measure} \label{3.1}
In this subsection, we establish the high-probability bound for the Hessian matrix $\nabla^2 \log p_t$ and use the high-probability bound to control the space-discretization error. This is the critical part of our analysis that allows us to prove Theorem~\ref{general}.
\begin{lemma}  \label{hessian-moment}
Let $P$ be a probability measure on $\RR^d$. Consider the density its Gaussian perturbation $p_\sigma(x) \propto \int_{\RR^d}\exp\left(-\frac{\|x-y\|^2}{2\sigma^2} \right) \d P(y)$. Then for $x \sim p_\sigma$, we have the sub-exponential norm bound
\begin{align*}
    \| \nabla^2 \log p_\sigma(x) \|_{F,\psi_1} \lesssim \frac{d}{\sigma^2},
\end{align*}
where $\|\cdot \|_{F,\psi_1} = \|\|\cdot \|_F \|_{\psi_1} $ denote the sub-exponential norm of the Frobenius norm of a random matrix.
\end{lemma}
\begin{proof}
Define the conditional density $\tilde{P}_\sigma(y|x)$ as $\d \tilde{P}_\sigma(y|x) \propto \exp\left(-\frac{\|y-x\|^2}{2\sigma^2} \right)\d P(y)$. Using Lemma \ref{computing-second-score}, $\nabla^2 \log p_\sigma$ can be written as
\begin{align*}
    \nabla^2 \log p_\sigma(x) = \mathrm{Var}_{\tilde{P}_\sigma(y|x)}\left(\frac{y}{\sigma^2} \right) - \frac{I_d}{\sigma^2}.
\end{align*}
For any positive integer $p$, using the fact that $\frac{y-x}{\sigma}$ is distributed as $\cN(0,I_d)$ and the power mean inequality,
\begin{align*}
    \EE_{p_\sigma(x)} \left\| \mathrm{Var}_{\tilde{P}_\sigma(y|x)} \left(\frac{y}{\sigma^2}\right)\right\|_F^p & \leq \frac{1}{\sigma^{2p}}\EE_{p_\sigma(x)} \left\|\EE_{\tilde{P}_\sigma(y|x)}\left(\frac{y-x}{\sigma}\right)\left(\frac{y-x}{\sigma}\right)^\top \right\|_F^p  \\
    & \leq \frac{1}{\sigma^{2p}} \EE_{z \sim \cN(0,I_d)}\|zz^\top \|_F^{p}. \\
    & \lesssim \left(\frac{pd}{\sigma^2}\right)^{p}.
\end{align*} 
Using the arbitrariness of $p$, we know that
$$\left\|\mathrm{Var}_{\tilde{P}_\sigma(y|x)} \left(\frac{y}{\sigma^2}\right) \right\|_{F,\psi_1} \lesssim \frac{d}{\sigma^2}.
$$
Thus by the triangle inequality,
$$\|\nabla^2 \log p_{\sigma}(x) \|_{F,\psi_1} \lesssim \frac{d}{\sigma^2}. $$
We complete the proof.
\end{proof}
\begin{lemma} \label{noLip}
There is a universal constant $K>0$ so that the following holds. 
For  $0 \leq t \leq s \leq T,\,\frac{s-t}{\sigma_t^2} \le\frac{1}{Kd}$, we have
\begin{align*}
   \EE \|\nabla \log {p}_t(x_t) - \nabla \log {p}_t(\alpha_{t,s}^{-1}x_s) \|^2 \lesssim  \frac{d^2(s-t)}{\sigma_t^4}.
\end{align*}
\end{lemma}
\begin{proof}
 We bound the difference between the value of $\nabla \log {p}_t$ at different points with the Hessian:
 \begin{align*}
  \nabla \log {p}_t({x}_t) - \nabla \log {p}_{t}(\alpha_{t,s}^{-1}x_s)  = \int_0^1 \nabla^2 \log p_t(x_t+a(\alpha_{t,s}^{-1}x_s - x_t))(\alpha_{t,s}^{-1}x_s - x_t) \d a.
\end{align*}
Thus 
\begin{align}
    \EE \|\nabla \log p_t(x_t) - \nabla \log p_t(\alpha_{t,s}^{-1}x_s) \|^2 \leq \int_0^1 \EE\left\|\nabla^2 \log p_t(x_t+az_{t,s})z_{t,s} \right\|^2 \d a, \label{122}
\end{align}
where $ z_{t,s}$ is defined by $z_{t,s} =  
\alpha_{t,s}^{-1}x_s - x_t \sim \cN\left(0,  (e^{s-t} - 1) I_d \right)$ and is independent of $x_t$. For random vectors $X,Y$, we use $P_{X,Y}$ to denote the joint probability measure of $(X,Y)$ and $P_{X|Y}$ to denote the conditional probability measure of $X$ given $Y$. Then for $0\le a\le 1$, we use change of measure to bound $\EE\left\|\nabla^2 \log p_t(x_t+az_{t,s})z_{t,s} \right\|^2$:
\begin{align} \label{123}
\begin{aligned}
 \EE\left\|\nabla^2 \log p_t(x_t+az_{t,s})z_{t,s} \right\|^2 &=  \EE \left[\left\|\nabla^2 \log p_t(x_t)z_{t,s} \right\|^2 \frac{\d P_{x_t+az_{t,s},z_{t,s}}(x_t,z_{t,s})}{\d P_{x_t,z_{t,s}}(x_t,z_{t,s})} \right] \\
 & \lesssim \left(\EE \left\|\nabla^2 \log p_t(x_t)z_{t,s} \right\|^4  \EE \left(\frac{\d P_{x_t+az_{t,s},z_{t,s}}(x_t,z_{t,s})}{\d P_{x_t,z_{t,s}}(x_t,z_{t,s})}  \right)^2 \right)^{1/2}. 
 \end{aligned}
\end{align}
Let $M_t =  \nabla^2 \log p_t(x_t)( \nabla^2 \log p_t(x_t))^\top,\,Z_{t,s} = z_{t,s} z_{t,s}^\top$. For $A,B \in \RR^{d\times d} $, define the tensor product $A \otimes B \in (\RR^d)^{\otimes 4}$ as $(A \otimes B)_{i_1,i_2,i_3,i_4} = A_{i_1i_2}B_{i_3i_4}$. Since $M_t$ and $Z_{t,s}$ are independent, the first factor in \eqref{123} can be written as
\begin{align*}
    \EE  \left\|\nabla^2 \log p_t(x_t)z_{t,s} \right\|^4  & = \EE\left[ \mathrm{Tr}\left(M_t^\top  Z_{t,s}\right)^2\right] \\
    & = \EE \langle M_t \otimes M_t, Z_{t,s} \otimes Z_{t,s} \rangle  \\
    & =  \langle\EE M_t \otimes M_t,  \EE Z_{t,s} \otimes Z_{t,s} \rangle.
\end{align*}
Notice that 
\begin{align*}
    \EE(Z_{t,s} \otimes Z_{t,s})_{i_1,i_2,i_3,i_4} = 
\begin{cases}
 3(e^{s-t} - 1)^2 ,\quad & i_1=i_2=i_3=i_4, \\
(e^{s-t} - 1)^2  ,\quad & i_1\neq i_2,\,(i_1,i_2)=(i_3,i_4)\, \text{or}\, (i_1,i_2)=(i_4,i_3), \\
0, \quad \text{else}.
\end{cases} 
\end{align*}
So we can bound the inner product by
\begin{align*}
    \langle\EE M_t \otimes M_t,  \EE Z_{t,s} \otimes Z_{t,s} \rangle & \lesssim (e^{s-t} - 1)^2 \left(\sum_{(i_1,i_2) =(i_3,i_4) } +\sum_{(i_1,i_2) =(i_4,i_3) }  \right) \EE(M_t \otimes M_t)_{i_1,i_2,i_3,i_4} \\
&  \lesssim (e^{s-t} - 1)^2  \sum_{(i_1,i_2) =(i_3,i_4) }  \EE(M_t \otimes M_t)_{i_1,i_2,i_3,i_4}  \\
&  \lesssim (e^{s-t} - 1)^2 \EE \|M_t\|_F^2 \\
& \lesssim  (e^{s-t}-1)^2  \EE \|\nabla^2 \log p_t(x_t) \|_F^4 \\
    & \lesssim (e^{s-t}-1)^2 \left(\frac{d}{\sigma_t^2}\right)^4.
\end{align*}
where the last inequality comes from Lemma \ref{hessian-moment}. Next, we bound the second term in \eqref{123}. By the data processing inequality, 
\begin{align*}
 \EE \left(\frac{\d P_{x_t+az_{t,s},z_{t,s}}(x_t,z_{t,s})}{\d P_{x_t,z_{t,s}}(x_t,z_{t,s})}  \right)^2 & = \EE \left(\frac{\d P_{x_t+az_{t,s}|z_{t,s}}(x_t|z_{t,s})}{\d P_{x_t|z_{t,s}}(x_t|z_{t,s})}  \right)^2 \\
 & \leq \EE \left(\frac{\d P_{x_t+az_{t,s}|z_{t,s},x_0}(x_t|z_{t,s},x_0)}{\d P_{x_t|z_{t,s},x_0}(x_t|z_{t,s},x_0)}  \right)^2 \\
 & = \EE \left(\frac{\d P_{x_t+az_{t,s}|z_{t,s},x_0}(x_t|z_{t,s},x_0)}{\d P_{x_t|x_0}(x_t|x_0)}  \right)^2.
 \end{align*}
 Notice that $x_t+az_{t,s} |(z_{t,s},x_0) \sim \cN(\alpha_t^{-1}x_0 + az_{t,s} , \sigma_t^2 I_d)$ and $x_t |x_0 \sim \cN(\alpha_t^{-1}x_0 , \sigma_t^2I_d )$. We can compute the chi-squared divergence explicitly:
 \begin{align*}
     \EE \left(\frac{\d P_{x_t+az_{t,s}|z_{t,s},x_0}(x_t|z_{t,s},x_0)}{\d P_{x_t|x_0}(x_t|x_0)}  \right)^2 = \EE \exp\left(\frac{a^2\|z_{t,s} \|^2 }{\sigma_t^2} \right)
 \end{align*}
 Finally, the condition $\frac{s-t}{\sigma_t^2} \le \frac{1}{Kd}$ implies $e^{s-t} -1 \lesssim s-t$ and $\frac{e^{s-t} -1}{\sigma_t^2} \lesssim \frac{1}{Kd}$. Thus for large enough $K$(actually, $K=1$ is enough),
 $$\EE \exp\left(\frac{a^2\|z_{t,s}\|^2}{\sigma_t^2} \right) = \left(1 - 2\frac{a^2 (e^{s-t} - 1)}{\sigma_t^2} \right)^{-d/2} \lesssim 1.$$
 Combining the bound for the first and the second terms of \eqref{123}, we conclude that
 \begin{align}
     \EE \|\nabla^2 \log p_t(x_t+az_{t,s})z_{t,s} \|^2 \lesssim \frac{d^2(s-t)}{\sigma_t^4}.\label{124}
 \end{align}
 Plugging \eqref{124} into \eqref{122}, we complete the proof. 
 \end{proof}
 \subsection{Stability of the Lipschitz Constant} \label{3.2}
In this subsection, we show that if $p_0$ satisfies the smoothness condition, $p_t$ is also smooth for sufficiently small $t$. In particular, under Assumption \ref{as5}, we can choose $t_0 \asymp \frac{1}{L}$ and an absolute constant $C$ such that for any $ 0\leq t \leq t_0$, the Lipschitz constant of $\nabla \log p_t$ is bounded by $CL$.
\begin{lemma}  \label{lownoise}
Suppose that Assumption \ref{as5} holds. If $\sigma_t^2 \leq \frac{\alpha_t}{2L}$, we have $\nabla \log p_t$ is $2L\alpha_t^{-1}$-Lipschitz on $\RR^d$.
\end{lemma}
\begin{proof}
Define a density $q(x) \propto p_0(\alpha_t^{-1} x)$. Then $\nabla \log q$ is $\alpha_t^{-1}L $-Lipschitz. Notice that $p_t$ is the Gaussian perturbation of $q$. Using Lemma \ref{computing-second-score}, we write the second-order score function of $p_t$ as
\begin{align*}
    \nabla^2 \log p_t(x) = \EE_{\tilde{q}_{\sigma_t} (y|x)} \nabla^2 \log q(y) + \mathrm{Var}_{\tilde{q}_{\sigma_t} (y|x)} (\nabla \log q(y)),
\end{align*}
where $\tilde{q}_{\sigma_t}(y|x)$ is the conditional density given by $\tilde{q}_{\sigma_t}(y|x) \propto q(y)\exp\left(\frac{\|x-y\|^2}{2\sigma_t^2} \right)$.
 When $\sigma_t^2 \leq \frac{\alpha_t}{2L}$, the conditional density satisfies $\log \tilde{q}_{\sigma_t}(y|x) = -\frac{y-x}{\sigma_t^2} +  \log q$ is $L\alpha_t^{-1} $-strongly concave, thus it satisfies the Poincar\'e inequality with a constant $\alpha_t L^{-1}$. From Lemma \ref{poin}, we obtain
    \begin{align*}
        \mathrm{Var}_{\tilde{q}_{\sigma_t}(y|x)}(\nabla \log q(y))    \preceq  \alpha_t L^{-1} \EE_{\tilde{q}_{\sigma_t}(y|x)} (\nabla^2 \log q(y))(\nabla^2 \log q(y))^\top \preceq L\alpha_t^{-1}I_d.  
    \end{align*}
Therefore, we have 
\begin{align*}
    \EE_{\tilde{q}_{\sigma_t} (y|x)} \nabla^2 \log q(y) + \mathrm{Var}_{\tilde{q}_{\sigma_t}(y|x)}(\nabla \log q(y)) 
 \preceq 2L\alpha_t^{-1} I_d.
\end{align*}
Meanwhile,
\begin{align*}
     \EE_{\tilde{q}_{\sigma_t} (y|x)}  \nabla^2 \log q(y) + \mathrm{Var}_{\tilde{q}_{\sigma_t}(y|x)}(\nabla \log q(y)) \succeq -L\alpha_t^{-1}I_d
\end{align*}
we complete the proof.
\end{proof}
\begin{lemma}\label{poin}
Let $P$ be a probability distribution on $\RR^d$ that satisfies a Poincar\'e inequality with constant $C_P$. For any function $f \in C^2(\mathrm{supp}(P))$, we have
\begin{align*}
   \mathrm{Var}_P(\nabla f)    \preceq C_P \EE_P (\nabla^2 f) (\nabla^2 f)^\top.
\end{align*}
\end{lemma}
\begin{proof}
For any vector $a \in \RR^d$, we have
\begin{align*}
    a^\top  \mathrm{Var}_P(\nabla f) a & \leq  \mathrm{Var}_P(a^\top \nabla f)  \\
    & \leq C_P \EE_P  \left\|\nabla (a^\top \nabla f) \right\|^2 \\
    & = C_P \EE_P \left\|(\nabla^2 f) a\right\|^2 \\
    & = C_P a^\top \EE_P (\nabla^2 f)(\nabla^2 f)^\top a. 
\end{align*}
We complete the proof.
\end{proof}
\section{Proofs for the Main Theorems}
\label{s:proofs}
Now we follow the discussion in Section \ref{overview} and combine everything together to complete the proof of our main theorems stated in Section \ref{Main-result}.
\subsection{Proof of Theorem \ref{smoothbound}} \label{proofsmoothbound}
\begin{lemma} \label{discrete-Lip}
For $t_{k-1} \leq t \leq t_k $, suppose that $\nabla \log p_t$ is $L$-Lipschitz for $t_{k-1} \leq t \leq t_k$. If $L \geq 1, h_k \leq 1$, we have
\begin{align*} 
     \EE \|\nabla \log p_t(x_t) - \nabla \log p_{t_k}(x_{t_k}) \|^2  \lesssim dL^2 (t_k-t)
\end{align*}
\end{lemma}
\begin{proof}
The space-discretization error is easily bounded by the Lipschitz condition:
\begin{align}\label{linear2}
\begin{aligned}  
    \EE \|\nabla \log p_t(x_t) - \nabla \log p_t(\alpha_{t,t_k}^{-1}x_{t_k}) \|^2 & \leq dL^2 \EE\|x_t -\alpha_{t,t_k}^{-1}x_{t_k} \|^2 \\ &=  dL^2(e^{t_k-t} - 1)  \\
    & \lesssim dL^2(t_k-t),                   
\end{aligned}
\end{align}
where the last inequality is because of $t_k-t \lesssim 1 $. Combining Lemma \ref{time-discrete}, Lemma \ref{score-bound}, and \eqref{linear2}, we have
\begin{align*} 
  &  \EE \|\nabla \log p_t(x_t) - \nabla \log p_{t_k}(x_{t_k}) \|^2 \\ &\lesssim    \EE \|\nabla \log p_t(x_t) - \nabla \log p_{t}(\alpha_{t,t_k}^{-1}x_{t_k}) \|^2 + \EE \|\nabla \log p_t(x_t) \|^2(1-\alpha_{t,t_k}^{-1})^2  \\
    & \lesssim dL^2 (t_k-t) + dL (t_k-t)^2 \\
    & \lesssim dL^2 (t_k-t).
\end{align*}
We complete the proof.
\end{proof}
\paragraph*{Proof of Theorem \ref{smoothbound}.}
 As shown in Section \ref{overview}, the extra terms arising in the discretization error of Euler-Maruyama scheme can be bounded by Lemma \ref{linear-discretization}, so we only need to consider the exponential integrator scheme. 
 By Proposition \ref{Girsanov}, we can bound the KL divergence between $p_0$ and $\hat{q}_T$ by 
 \begin{align} \label{Girsanov3}
\mathrm{KL}(p_0 \| \hat{q}_T) \lesssim \mathrm{KL}(p_T\| \gamma_d) + T\epsilon_0^2 + \sum_{k=1}^N \int_{t_{k-1}}^{t_k} \EE \|\nabla \log p_t(x_t) - \nabla \log p_{t_k}(x_{t_k}) \|^2 \d t.
 \end{align}
The first term in \eqref{Girsanov3} is bounded by Lemma \ref{converge-forward}. Then, we apply Lemma \ref{discrete-Lip} to bound the discretization error:
\begin{align*}
& \sum_{k=1}^N \int_{t_{k-1}}^{t_k}\EE \|\nabla \log p_t(x_t) - \nabla \log p_{t_k}(x_{t_k}) \|^2 \d t \\
& \lesssim \sum_{k=1}^N dL^2 \int_{t_{k-1}}^{t_k}(t_k-t) \d t \\
& \lesssim dL^2 \sum_{k=1}^N h_k^2.
\end{align*}
For uniform discretization, the above quantity is $\frac{dT^2L^2}{N}$. We complete the proof.  \qed

\subsection{Proof of Theorem \ref{general}}
\begin{lemma} \label{early-discrete}
There is a constant $K$ such that the following holds. 
In the early stopping setting, suppose that the variance function $g$ satisfies $\frac{h_k}{\sigma_{t_{k-1}}^2} \le \frac{1}{Kd} $ for any integer $1 \leq k \leq N$. Then we have
 \begin{align*}
    \sum_{k=1}^N \int_{t_{k-1}}^{t_k}  \EE \|\nabla \log p_t(x_t) - \nabla \log p_{t_k}(x_{t_k}) \|^2 \d t \lesssim d^2\sum_{k=1}^N\frac{h_k^2}{\sigma_{t_{k-1}}^4} 
 \end{align*}
\end{lemma}
\begin{proof}
By Lemma \ref{time-discrete} and Lemma \ref{score-bound}, we have
\begin{align}
\begin{aligned} \label{555}
  &  \EE \|\nabla \log p_t(x_t) - \nabla \log p_{t_k}(x_{t_k}) \|^2 \\ & \lesssim    \EE \|\nabla \log p_t(x_t) - \nabla \log p_{t}(\alpha_{t,t_k}^{-1}x_{t_k}) \|^2 + \EE \|\nabla \log p_t(x_t) \|^2(1-\alpha_{t,t_k}^{-1})^2 \\  
    & \lesssim  \EE \|\nabla \log p_t(x_t) - \nabla \log p_{t}(\alpha_{t,t_k}^{-1}x_{t_k}) \|^2 + \frac{d(1-\alpha_{t,t_k}^{-1})^2}{\sigma_t^2}.
\end{aligned}
\end{align}
From Lemma \ref{noLip} we have
\begin{align} \label{666}
    \EE \|\nabla \log p_t(x_t) - \nabla \log p_t(\alpha_{t,t_k}^{-1}x_{t_k}) \|^2 \lesssim \frac{d^2(t_k-t)}{\sigma_t^4}.
\end{align}
Noticing that $\frac{h_k}{\sigma_{t_{k-1}}^2} \lesssim \frac{1}{d}$ implies $\frac{(1-\alpha_{t,t_k}^{-1})^2}{\sigma_t^2} \lesssim  \frac{t_k-t}{d}$ and combining this with \eqref{555} and \eqref{666}, we conclude that
\begin{align*}
    \EE \|\nabla \log p_t(x_t) - \nabla \log p_{t_k}(x_{t_k}) \|^2 \lesssim \frac{d^2(t_k-t)}{\sigma_t^4}.
\end{align*}
Therefore,
\begin{align*}
 & \sum_{k=1}^N \int_{t_{k-1}}^{t_k} \EE \|\nabla \log p_t(x_t) - \nabla \log p_{t_k}(x_{t_k}) \|^2 \d t \\ & \lesssim \sum_{k=1}^N  \int_{t_{k-1}}^{t_k} \frac{d^2(t_k-t)}{\sigma_t^4}\d t\\  & \lesssim  \sum_{k=1}^N  \int_{t_{k-1}}^{t_k} \frac{d^2(t_k-t)}{\sigma_{t_{k-1}}^4}\d t \\ &  \lesssim d^2\sum_{k=1}^N \frac{h_k^2}{\sigma_{t_{k-1}}^4}.
\end{align*}
We complete the proof.
\end{proof}

 \begin{lemma}\label{quant-exp}
If $K\ge 2$, $c\le \frac{1}{Kd}$, $t_0=\de$, $t_N=T$, and
$h_k:=t_k-t_{k-1}=c\min\{t_k,1\}$, then $\frac{h_k}{\sigma_{t_k}^2} \lesssim \frac{1}{Kd}$ for $k=1,\ldots,N$ and
\[
\Pi:= \sum_{k=1}^N\fc{h_k^2}{\si_{t_{k-1}}^4} \lesssim  c\pa{\log\rc\de + T}.
\]
\end{lemma}
\begin{proof}
Note that $\sigma_t^2 \asymp \max\{1,t\} $.
We consider the sum with $t_k\le 1$ and $t_k>1$ separately. For $t_k\le 1$, we have $\fc{G_k}{\sigma_{k-1}^2} \asymp  \fc{ct_k}{t_{k-1}} \le \frac{2}{Kd}$ (when $K\ge 2$, so $\fc{t_k}{t_{k-1}}\le 2$). Noting that the number of terms in the sum is $\lesssim \log_{1-c}(\de)$, 
\begin{align}\label{e:exp1}
     \sum_{k:t_k\le 1} \fc{h_k^2}{\min\{t_{k-1}^2,1\}}
     &=     \sum_{k:t_k\le 1} \fc{c^2t_k^2}{t_{k-1}^2} \asymp 
     c^2 \log_{1-c}(\de)\asymp c^2 \fc{\log(1/\de)}{c}=c\log(1/\de).
\end{align}
For $t_k>1$, $\fc{h_k}{\min\{t_{k-1},1\}} = c \le \frac{1}{Kd}  $  and
\begin{align}\label{e:exp2}
    \sum_{k:t_k> 1} \fc{h_k^2}{\min\{t_{k-1}^2,1\}}
     &= \sum_{k:t_k> 1}c^2 \lesssim  c^2 \cdot \fc Tc = cT.
\end{align}
Combining~\eqref{e:exp1} and~\eqref{e:exp2} gives the result. Note the number of steps is 
\[N\lesssim\log_{1-c}(\de) + \fc Tc = \rc c(\log \de + T). \]
\end{proof}

 \paragraph*{Proof of Theorem \ref{general}.}
 As shown in Section \ref{overview}, the extra terms arising in the discretization error of the Euler-Maruyama scheme can be bounded by Lemma \ref{linear-discretization}, so we only need to consider the exponential integrator scheme.
 From Proposition \ref{Girsanov} we obtain
 \begin{align} \label{earlyGrisanov1}
    \mathrm{KL}(p_{t_0} \| \hat{q}_{T-t_0}) \lesssim \mathrm{KL}(p_T\| \gamma_d) + \sum_{k=1}^N \int_{t_{k-1}}^{t_k} \EE \| \nabla \log p_t(x_t) - \nabla \log p_{t_k}(x_{t_k})\|^2 \d t + T\epsilon_{0}^2.
\end{align}
By bounding the first term in \eqref{earlyGrisanov1} with Lemma \ref{converge-forward} and the second term in \eqref{earlyGrisanov1} with Lemma \ref{early-discrete}, we obtain \eqref{e:kl-ineq}. Further more, we can further quantify the term $\Pi = \sum_{k=1}^N \frac{G_k^2}{\sigma_{t_{k-1}}^4}$ for exponentially decaying
(and then constant) step size with Lemma \ref{quant-exp}.
\qed

\subsection{Proof of Corollary \ref{cor:main} and Corollary \ref{cor:W2}}
\paragraph*{Proof of Corollary \ref{cor:main}.}
We use the exponentially decreasing step size in Theorem~\ref{general}. We note that $W_2(P,M_\sharp p_{\delta}) \lesssim \sqrt{d\si_\delta^2}\asymp \sqrt{d\delta} $, so choose $\de \asymp \frac{\ew^2}{d}$. 
Choose $T \asymp \log\pf{d+M_2}{\ep_{\KL}^2}$. Also choose $c \asymp \fc{\ep_{\KL}^2}{d^2\pa{T + \log \prc\de}} \gtrsim \fc{\ep_{\KL}^2}{d^2\log\pf{(d+M_2)d}{\ep_{\KL}^2\ew^2}}$. 
If $\ep_0^2 \lesssim \frac{\ep_{\KL}^2}{T^2}$, this ensures that all terms are $\lesssim \ep_{\KL}^2$. Choosing appropriate implied constants completes the proof.  \qed

\begin{lemma}\cite[Lemma 6.6]{convergencescore2}\label{gaussian-tail-calculation} Let $\mu$ be the standard Gaussian measure on $N\left(0, I_d\right)$. Then
$$
\begin{aligned}
 \sup _{\mu(A) \leq \ep} \int_A\|x\|^2 \mu(d x) \leq \ep\left(2 d+3 \ln \left(\frac{1}{\ep}\right)+3\right)=O\left(\ep\left(d+\ln \left(\frac{1}{\ep}\right)\right)\right).
\end{aligned}
$$
\end{lemma}
\begin{proof}
By the $\chi^2$ tail bound in \cite{Laurent2000AdaptiveEO}, for $t \geq 0$
$$
\mu\left(\|X\|^2 \geq 2 d+3 t\right) \leq \mathbb{P}\left(\|X\|^2 \geq d+2 \sqrt{d t}+2 t\right) \leq e^{-t},
$$
so $\|X\|^2$ is stochastically dominated by a random variable with cdf $F(y)=1-e^{-\frac{y-2 d}{3}}$. Then letting $P_Y$ be the measure corresponding to $F$,
$$
\begin{aligned}
\sup _{\mu(A) \leq \ep} \int_A\|x\|^2 \mu(d x) & \leq \sup _{P_Y(A) \leq \ep} \int_A y P_Y(d y)=\int_{2 d+3 \ln \left(\frac{1}{\ep}\right)}^{\infty} y \,d F(y) \\
& =\ep\left(2 d+3 \ln \left(\frac{1}{\ep}\right)\right)+\int_{2 d+3 \ln \left(\frac{1}{\ep}\right)}^{\infty} e^{-\frac{y-2 d}{3}} d y=\ep\left(2 d+3 \ln \left(\frac{1}{\ep}\right)\right)+3 \ep
\end{aligned}
$$
\end{proof}
\paragraph*{Proof of Corollary \ref{cor:W2}.}
Let $p_{\delta}^{\mathrm{trunc}}$ be the law of $x_\delta^{\mathrm{trunc}}:=x_\delta 1_{\{x_\delta \in B_R(0)\}}$ and define $  \hat{q}_{T-\delta}^{\mathrm{trunc}}$ similarly.
Note that 
\begin{align}
\begin{aligned}
W_2(P, M_\sharp \hat{q}_{T-\delta}^{\mathrm{trunc}}) & \leq W_2(P, M_\sharp 
 p_\delta) + W_2(M_\sharp  p_\delta, M_\sharp 
 \hat{q}_{T-\delta}^{\mathrm{trunc}} ) \\
& \lesssim \sqrt{d \delta}  + W_2(p_\delta, \hat{q}_{T-\delta}^{\mathrm{trunc}} ). 
\end{aligned}
\label{secondterm}
\end{align}
To bound the second term in \eqref{secondterm}, we consider a coupling $x_\delta \sim p_\delta$ and $ \hat{y}_{T-\delta}^{\mathrm{trunc}} \sim \hat{q}_{T-\delta}^{\mathrm{trunc}}$ such that $x_\delta \ne \hat{y}_{T-\delta}^{\mathrm{trunc}}$ with probability $\epsilon_{\mathrm{TV}}$, where 
\begin{align}
\label{e:w2pf-1}
     \epsilon_{\mathrm{TV}} := \mathrm{TV}(p_\delta, \hat{q}_{T-\delta}^{\mathrm{trunc}}) & \leq  \mathrm{TV}(p_\delta^{\mathrm{trunc}}, \hat{q}_{T-\delta}^{\mathrm{trunc}}) +  \mathrm{TV}(p_{\delta}, p_{\delta}^{\mathrm{trunc}}) \\
     \label{e:w2pf-2}
     & \leq  \mathrm{TV}(p_\delta, \hat{q}_{T-\delta}) + 
     \mathrm{TV}(p_{\delta}, p_{\delta}^{\mathrm{trunc}}) 
      \\&
     \leq \sqrt{\KL(p_\delta \| \hat{q}_{T-\delta})} + 
     \PP\left( \|{x}_{\delta}\| \geq R\right) 
     \label{e:w2pf-3}
     \\
     &= \tilde O(\ep_0) + \PP\left( \|{x}_{\delta}\| \geq R\right) .\label{TV-bound}
\end{align}
We used the triangle inequality, data processing inequality, and Pinsker's inequality in~\eqref{e:w2pf-1},~\eqref{e:w2pf-2}, and~\eqref{e:w2pf-3}, respectively.  
Express ${x}_\delta=\alpha_\delta {x}_0+\sigma_\delta \xi$, where ${x}_0 \sim P,\, \xi \sim \cN(0,I_d)$. Now
\begin{align} \label{coupling}
\begin{aligned}
\EE\left\|{x}_\delta-\hat{y}_{T-\delta}^{\mathrm{trunc}}\right\|^2 & \leq \sup _{P(A) \leq \ep_{\mathrm{TV}}} 2\left(\mathbb{E}\left[\left\|\al_\delta {x}_0-\hat{y}_{T-\delta}^{\mathrm{trunc}}\right\|^2{1}_A\right]+\sigma_\delta^2 \mathbb{E}\left[\|\xi\|^2 {1}_A\right]\right) \\
& \leq 2\left((2M_2+2R^2)\ep_{\mathrm{TV}}+\sigma_\delta^2 \ep_{\mathrm{TV}} \cdot O\left(d+\log 
\left(\frac{1}{\ep_{\mathrm{TV}}}\right)\right)\right),
\end{aligned}
\end{align}
where the second inequality comes from Lemma \ref{gaussian-tail-calculation}. Combining \eqref{secondterm}, \eqref{TV-bound}, \eqref{coupling} and the choice of parameters in \eqref{parameter}, we complete the proof.
\qed

\subsection{Proof of Theorem \ref{truncating}}
\paragraph*{Proof of Theorem \ref{truncating}.}
 As shown in Section \ref{overview}, the extra terms arising in the discretization error of the Euler-Maruyama scheme can be bounded by Lemma \ref{linear-discretization}, so we only need to consider the exponential integrator scheme. 
 Using Proposition \ref{Girsanov}, we obtain
\begin{align}
\begin{aligned} \label{truncateGirsanov}
 \mathrm{KL}({p}_0 \| \hat{q}_T ) & \lesssim  \mathrm{KL}({p}_T \| \gamma_d )  + \sum_{k=1}^N \int_{t_{k-1}}^{t_k} \EE\|\nabla \log {p}_t(x_t) - \nabla \log {p}_{t_k}(x_{t_k}) \|^2 \d t  + T\epsilon_0^2.
 \end{aligned}
\end{align}
  In the right hand side of \eqref{truncateGirsanov}, the first term is directly bounded by Lemma \ref{converge-forward}.  Thus we only have to consider the second term, which is the discretization error. 
 Let $k_0$ be the largest index such that $t_{k_0}\le \rc L$. By Lemma~\ref{early-discrete} and Lemma~\ref{quant-exp}, 
 \begin{align*}
     \sum_{k = k_0 + 1}^N \int_{t_{k-1}}^{t_k} \EE \|\nabla \log p_t(x_t) - \nabla \log p_{t_k}(x_{t_k}) \|^2 \d t & \lesssim d^2\sum_{k=k_0+1}^{N}\frac{h_k^2}{\sigma_{t_{k-1}}^4} \lesssim d^2c(\log L + T).
 \end{align*}
 The number of steps for this part is $N-k_0\lesssim \rc c(\log L+T)$.
 Note $k_0\lesssim \rc c$ so by Lemma \ref{discrete-Lip} and Lemma \ref{lownoise},
 \begin{align*}
     \sum_{k=1}^{k_0} \int_{t_{k-1}}^{t_k}\EE \|\nabla \log {p}_t(x_t) - \nabla \log {p}_{t_k}(x_{t_k})  \|^2 & \lesssim dL^2 \sum_{k=1}^{k_0} h_k^2
     \lesssim dL^2 \cdot \rc c \pf{c}{L}^2 = cd.
 \end{align*}
 Thus the total discretization error is bounded by 
 \begin{align*}
     \sum_{k=1}^{N}  \int_{t_{k-1}}^{t_k}\EE \|\nabla \log {p}_t(x_t) - \nabla \log {p}_{t_k}(x_{t_k})  \|^2 \d t\lesssim 
     d^2c(\log L + T)
 \end{align*}
 and the total number of steps is $N\lesssim \rc c (\log L+T)$. Given the number of steps $N$, we can choose $c = \fc{\log L + T}{N}$; plugging this in gives the bound.
 We complete the proof.
\section{Lemmas for Computing Score Functions} \label{computation-of-score}
In this section, we provide some lemmas for the score function, which will be used in our analysis.\qed
\begin{lemma} \label{computing-score}
Let $P$ be a probability measure on $\RR^d$. Consider the Gaussian perturbation of $P$ that admits a density
           $p_{\mu,\sigma}(x) \propto \int_{\RR^d} \exp\left(-\frac{\|x-\mu y\|^2}{2\sigma^2} \right) \d P(y) $. Let $\tilde{P}_{\mu,\sigma}(y|x) $ be the conditional probability measure satisfying $\d \tilde{P}_{\mu,\sigma}(y|x) \propto \exp\left(-\frac{\|x-\mu y\|^2}{2\sigma^2} \right) \d P(y) $.
\begin{enumerate}
    \item If $P$ admits a density $p \in C^1(\RR^d)$, we have
     $$\nabla \log p_{\mu,\sigma}(x) = \frac{1}{\mu} \EE_{\tilde{P}_{\mu,\sigma}(y|x)} \nabla_y \log p(y).   $$
    \item We have
    $$       \nabla \log p_{\mu,\sigma}(x)     =  \EE_{\tilde{P}_{\mu,\sigma}(y|x)} \Bigl(\frac{\mu y-x}{\sigma^2}\Bigr).           $$
\end{enumerate}
\end{lemma}
\begin{proof}
The first expression is obtained by
\begin{align*}
   \nabla \log p_{\mu,\sigma}(x) &= \frac{\int_{\RR^d}  p(y) \nabla_x \left[ \exp\left(-\frac{\|x - \mu y\|^2}{2\sigma^2}  \right)\right] \d y}{\int_{\RR^d}  p(y)\exp\left(-\frac{\|x - \mu y\|^2}{2\sigma^2}  \right) \d y} \\
   & = -\frac{\int_{\RR^d}  p(y) \nabla_y\left[ \exp\left(-\frac{\|x - \mu y\|^2}{2\sigma^2}  \right)\right]\d y}{\mu \int_{\RR^d}  p(y)\exp\left(-\frac{\|x - \mu y\|^2}{2\sigma^2}  \right) \d y}\\ 
   & = \frac{\int_{\RR^d}  \nabla_y p(y)  \exp\left(-\frac{\|x - \mu y\|^2}{2\sigma^2}  \right) \d y}{\alpha_{t,s} \int_{\RR^d}  p(y)\exp\left(-\frac{\|x - \mu y\|^2}{2\sigma^2 }  \right) \d y} \\
   & = \frac{1}{\mu}\EE_{\tilde{P}_{\mu,\sigma}(y|x)} \nabla_y \log p(y).
\end{align*}
For the second expression, 
\begin{align*}
     \nabla \log p_{\mu,\sigma}(x) &= \frac{\int_{\RR^d}   \nabla_x \left[ \exp\left(-\frac{\|x - \mu y\|^2}{2\sigma^2}  \right)\right] \d P( y)}{\int_{\RR^d}  \exp\left(-\frac{\|x - \mu y\|^2}{2\sigma^2}  \right) \d P(y)} \\
     &  = \frac{\int_{\RR^d}     \frac{\mu y-x}{\sigma^2} \exp\left(-\frac{\|x - \mu y\|^2}{2\sigma^2}  \right) \d P( y)}{\int_{\RR^d}  \exp\left(-\frac{\|x - \mu y\|^2}{2\sigma^2}  \right) \d P(y)}       \\
     & = \EE_{\tilde{P}_{\mu,\sigma}(y|x)} \Bigl(\frac{\mu y-x}{\sigma^2}\Bigr). \qedhere
\end{align*}
\end{proof}
\begin{lemma} \label{score-bound}
 Let $p \in C^1(\RR^d)$ be a probability density.
 \begin{enumerate}
     \item \cite{pmlr-v178-chewi22a} If $\nabla \log p$ is $L$-Lipchitz, we have
     \begin{align*}
         \EE_p \|\nabla \log p(x) \|^2 \leq dL.
     \end{align*}
     \item If there exists a probability measure $Q$ and $\sigma >0$ such that $p(x) \propto \int_{\RR^d} \exp\left(-\frac{\|x -y \|^2 }{2\sigma^2}\right) \d Q(y)$.
     then $ \EE_p \|\nabla \log p(x) \|^2 \leq \frac{d}{\sigma^2}$.
 \end{enumerate}
\end{lemma}
\begin{proof}
\begin{enumerate}
    \item Using integration by parts, we have
\begin{align*}
    \EE_p \|\nabla \log p \|^2 & = \int_{\RR^d} p(x) \|\nabla \log p(x) \|^2 \d x   \\
    & = \int_{\RR^d} \langle \nabla  p(x),  \nabla \log p(x) \rangle \d x \\
    & = \int_{\RR^d} p(x) \Delta \log p(x) \d x \\
    & \leq dL.
\end{align*}
    \item Using Lemma \ref{computing-score}, we rewrite the score function as
    \begin{align*}
       \nabla \log p(x) = \EE_{\tilde{Q}_\sigma(y|x)} \Bigl(\frac{y-x}{\sigma^2}\Bigr),
    \end{align*}
    where $\tilde{Q}_\sigma$ is the conditional density $\d \tilde{Q}_\sigma(y|x) \propto \exp\left(-\frac{\|x-y \|^2}{2\sigma^2} \right) \d Q(y)$. Then the second moment of the score function is bounded by
    \begin{align*}
        \EE_p \|\nabla \log p(x) \|^2 = \EE_{p(x)} \left\| \EE_{\tilde{Q}_\sigma(y|x)}\Bigl(\frac{y-x}{\sigma^2}\Bigr) \right\|^2 \leq \EE_{p(x)}\EE_{\tilde{Q}_\sigma(y|x)} \left\|\frac{y-x}{\sigma^2} \right\|^2 \leq \frac{d}{\sigma^2}.
    \end{align*}
\end{enumerate}
\end{proof}
\begin{lemma} \label{computing-second-score}
Let $P$ be a probability measure on $\RR^d$. Consider the density of its Gaussian perturbation $p_{\sigma}(x) \propto \int_{\RR^d}\exp\left(-\frac{\|x-y\|^2}{2\sigma^2} \right) \d P(y)$. Define a conditional probability measure $\tilde{P}_\sigma(y|x)$ as $\d \tilde{P}_\sigma(y|x) \propto \exp\left(\frac{\|x-y\|^2}{2\sigma^2} \right)\d P(y)$.
\begin{enumerate}
    \item If $P$ admits a density $p \in C^2(\RR^d)$, we have
\begin{align*}
    \nabla^2 \log p_{\sigma}(x) = \EE_{\tilde{P}_{\sigma}(y|x)} \nabla^2 \log p(y) + \mathrm{Var}_{\tilde{P}_{\sigma}(y|x)}(\nabla \log p(y) ).
\end{align*}
\item We have
 $$  \nabla^2 \log p_{\sigma}(x) = \mathrm{Var}_{\tilde{P}_{\sigma}(y|x)}\left(\frac{y}{\sigma^2} \right) - \frac{I_d}{\sigma^2}.$$
 \end{enumerate}
\end{lemma}
\begin{proof} 
We rewrite the second-order score function as
\begin{align*}
    \nabla^2 \log p_\sigma(x) = \frac{\nabla^2 p_\sigma(x)}{p_\sigma(x)} - \nabla \log p_\sigma(x) (\nabla \log p_\sigma(x)  )^\top.
\end{align*}
To prove the first expression, we write
\begin{align*}
    \frac{\nabla^2 p_\sigma(x)}{p_\sigma(x)} & = \frac{\int p(y) \nabla_x^2 \exp\left(\frac{-\|x-y\|^2}{2\sigma^2} \right)\d y}{\int p(y)\exp\left(\frac{-\|x-y\|^2}{2\sigma^2} \right) \d y } \\
    & = \frac{\int p(y) \nabla_y^2 \exp\left(\frac{-\|x-y\|^2}{2\sigma^2} \right) \d y}{\int  p(y)\exp\left(\frac{-\|x-y\|^2}{2\sigma^2} \right)\d y } \\
    & = \frac{\int \exp\left(\frac{-\|x-y\|^2}{2\sigma^2} \right) \nabla_y^2 p(y)  \d y}{\int  p(y)\exp\left(\frac{-\|x-y\|^2}{2\sigma^2} \right) \d y }  \\
    & = \EE_{\tilde{P}_\sigma(y|x)} \frac{\nabla_y^2 p(y)}{p(y)},
\end{align*}
It follows from Lemma \ref{computing-score} that
\begin{align*}
\nabla \log p_{\sigma}(x) =  \EE_{\tilde{P}_{\sigma}(y|x)} \nabla_y \log p(y).
\end{align*}
Combining the two terms, we arrive at
\begin{align*}
    \nabla^2 \log p_\sigma(x) &= \EE_{\tilde{P}_\sigma(y|x)} \frac{\nabla_y^2 p(y)}{p(y)} - \EE_{\tilde{P}_{\sigma}(y|x)} \nabla_y \log p(y)\left(\EE_{\tilde{P}_{\sigma}(y|x)} \nabla_y \log p(y) \right)^\top \\
     & = \EE_{\tilde{P}_\sigma(y|x)} \nabla_y^2 \log p(y) + \mathrm{Var}_{\tilde{P}_\sigma(y|x)}(\nabla \log p(y)).
\end{align*}

To prove the second expression, we note that
\begin{align*}
      \frac{\nabla^2 p_\sigma(x)}{p_\sigma(x)} & = \frac{\int  \nabla_x^2 \exp\left(-\frac{\|x-y\|^2}{2\sigma^2} \right)\d P(y)}{\int \exp\left(-\frac{\|x-y\|^2}{2\sigma^2} \right) \d P(y)} \\
      & =  \frac{\int \exp\left(-\frac{\|x-y\|^2}{2\sigma^2}\right) \left(\frac{(x-y)(x-y)^\top}{\sigma^4} - \frac{I_d}{\sigma^2} \right) \d P(y) }{\int \exp\left(-\frac{\|x-y\|^2}{2\sigma^2} \right) \d P(y) } \\
      & = \EE_{\tilde{P}_\sigma (y|x)} \left(\frac{(x-y)(x-y)^\top}{\sigma^4} - \frac{I_d}{\sigma^2} \right).
\end{align*}
It follows from Lemma \ref{computing-score} that
\begin{align*}
\nabla \log p_{\sigma}(x) &=  \EE_{\tilde{P}_{\sigma}(y|x)} \Bigl(\frac{ y-x}{\sigma^2}\Bigr).
\end{align*}
Combining the two terms, we have
\begin{align*}
      \nabla^2 \log p_\sigma(x) &= \EE_{\tilde{P}_\sigma (y|x)} \left(\frac{(x-y)(x-y)^\top}{\sigma^4} - \frac{I_d}{\sigma^2} \right) - \EE_{\tilde{P}_\sigma (y|x)} \frac{y-x}{\sigma^2}  \left(\EE_{\tilde{P}_\sigma (y|x)} \frac{y-x}{\sigma^2} \right)^\top  \\
      & = \mathrm{Var}_{\tilde{P}_\sigma (y|x)} \left(\frac{y}{\sigma^2} \right) - \frac{I_d}{\sigma^2}.  \qedhere
\end{align*}
\section{Technical details for Proposition \ref{Girsanov}} \label{C}
\paragraph*{Proof of Lemma \ref{diff}.}
By the Fokker-Plank equation, the evolution of $p_t$ and $q_t$ is given by
\begin{align}
    \pdt{p_t}(x) &= \nabla \cdot \left[ -F_1(x,t)p_t(x) +  \frac{g(t)^2}{2} \nabla p_t(x)\right]\\
    \pdt{q_t}(x)& = \nabla \cdot \left[-F_2(x,t)q_t(x) + \frac{g(t)^2}{2}\nabla q_t(x) \right]
\end{align}
Then we have
$$\dt \KL(p_t\| q_t)= \int \log \frac{p_t}{q_t}  \pdt{p_t} \d x - \int \frac{p_t}{q_t} \pdt{q_t} \d x.$$
For the first term,
\begin{align*}
    \int \log \frac{p_t}{q_t}  \pdt{p_t} \d x &= \int  \nabla \cdot \left[-p_t(x)F_1(x,t)+ \frac{g(t)^2}{2}\nabla p_t(x) \right] \log \frac{p_t(x)}{q_t(x)} \d x   \\
    &= \int \left\langle \nabla \log  \frac{p_t(x)}{q_t(x)}, p_t(x)F_1(x,t) - \frac{g(t)^2}{2}\nabla p_t(x)  \right\rangle \d x \\
    &= \int p_t(x)\left \langle F_1(x,t), \nabla \log \frac{p_t(x)}{q_t(x)} \right \rangle \d x - \int \frac{g(t)^2}{2} \left\langle \nabla \log \frac{p_t(x)}{q_t(x)}, \nabla p_t(x) \right\rangle \d x
\end{align*}
For the second term,
\begin{align*}
    \int \frac{p_t}{q_t} \pdt{q_t} \d x &= \int \frac{p_t}{q_t}  \nabla \cdot \left[ -F_2(x,t)q_t(x) +  \frac{g(t)^2}{2} \nabla q_t(x)\right] \d x \\
    &= \int \left \langle  \nabla \frac{p_t}{q_t}, F_2(x,t)q_t(x) - \frac{g(t)^2}{2}\nabla q_t(x) \right \rangle \d x \\
    & = \int q_t(x)\left \langle  \nabla \frac{p_t}{q_t}, F_2(x,t)  \right \rangle \d x  - \frac{g(t)^2}{2} \left  \langle  \nabla \frac{p_t}{q_t}, \nabla q_t(x) \right \rangle \d x.
\end{align*}
Notice that 
\begin{align*}
 &   \int  \left \langle  \nabla \frac{p_t}{q_t}, \nabla q_t(x) \right \rangle \d x - \int \left\langle \nabla \log \frac{p_t}{q_t}, \nabla p_t(x) \right\rangle \d x \\ &=\int \left \langle  \frac{q_t\nabla p_t  - p_t\nabla q_t}{q_t}, \nabla \log q_t \right \rangle \d x - \int p_t \left\langle \nabla \log \frac{p_t}{q_t}, \nabla \log p_t(x) \right\rangle \d x  \\
 &= \int p_t \left \langle   \nabla \log \frac{p_t}{q_t}, \nabla \log q_t \right \rangle \d x - \int p_t \left\langle \nabla \log \frac{p_t}{q_t}, \nabla \log p_t(x) \right\rangle \d x   \\
 & = -J(p_t\| q_t),
\end{align*}
and 
\begin{align*}
 &   \int p_t(x)\left \langle F_1(x,t), \nabla \log \frac{p_t}{q_t} \right \rangle \d x -  \int q_t(x)\left \langle  \nabla \frac{p_t}{q_t}, F_2(x,t)  \right \rangle \d x \\
 & = \int p_t(x)\left \langle F_1(x,t), \nabla \log \frac{p_t}{q_t} \right \rangle \d x -  \int p_t(x)\left \langle  \nabla \log \frac  {p_t}{p_t}, F(x,t)  \right \rangle \d x \\
 & = \int p_t(x) \left \langle \nabla 
 \log \frac{q_t}{p_t}, F_1(x,t) - F_2(x,t) \right \rangle  \\
 & = \EE \left[ \left \langle F_1(X_t,t) - F_2(X_t,t) , \nabla \log \frac{q_t(X_t)}{p_t(X_t)} \right \rangle \right].
\end{align*}
We complete the proof.
\qed
\paragraph*{Proof of Lemma \ref{tech}(\ref{1.}).}
 The uniqueness and regularity for the discrete interpolation \eqref{discrete-local} are obvious since the drift term is linear. Now we check the uniqueness and regularity for \eqref{backward-local}. In fact, the uniqueness of \eqref{backward-local} is guaranteed by the local Lipschitz property of $\nabla \log \tilde{p}_t$ (see, e.g., \cite[Chapter 5, Theorem 2.5]{Karatzas1987BrownianMA}) since $\tilde{p}_t \in C^2(\RR^d)$ is supported on $\RR^d$. For the regularity, we note that
 \begin{align*}
     \tilde{p}_{t|t_k'}(x|a) = p_{T-t|T-t_k'}(x|a) = \frac{p_{T-t}(x)p_{T-t_k'|T-t}(a|x)}{p_{T-t_k'}(a)},
 \end{align*}  
where $p_{t_1|t_2}$ is the conditional
density of $x_{t_1}$ given $x_{t_2} $. Since $p_{T-t_k'|T-t}(a|x)$ has distribution $\mathcal N(\alpha_{T-t,T-t_k'}x,(1-\alpha_{T-t,T-t_k'}^2)I_d)$, it is smooth for any $a\in \RR^d$, and we have $ \tilde{p}_{t|t_k'}(x|a) \in C^2(\RR^d)$.
\qed

\medskip 

In order to prove Lemma \ref{tech}(\ref{2.}), we need the following.
\begin{lemma} \label{score-sub}
Let $P$ be a probability measure on $\RR^d$. Consider the Gaussian perturbation of $P$ that admits a density
          $p_{\sigma}(x) \propto \int_{\RR^d} \exp\left(-\frac{\|x- y\|^2}{2\sigma^2} \right) \d P(y) $. Let $\tilde{P}_{\sigma}(y|x) $ be the conditional probability measure satisfying $\d \tilde{P}_{\mu,\sigma}(y|x) \propto \exp\left(-\frac{\|x-y\|^2}{2\sigma^2} \right) \d P(y) $. For $x \sim p_\sigma$ we have
    $$       \|\nabla \log p_{\sigma}(x)\|_{\psi_2} \lesssim \sqrt{\frac{d}{\sigma^2}}. $$
\end{lemma}
\begin{proof}
By Lemma \ref{computing-score}, we write the score function of $p_\sigma$ as
\begin{align*}
      \nabla \log p_\sigma(x) = \EE_{\tilde{P}_\sigma(y|x)}\left(\frac{y-x}{\sigma^2}\right),
\end{align*}
where $\tilde{P}_\sigma(y|x)$ be the conditional probability measure satisfying $\d \tilde{P}_{\mu,\sigma}(y|x) \propto \exp\left(-\frac{\|x-\mu y\|^2}{2\sigma^2} \right) \d P(y) $. For any positive integer $p$, using the fact that $\frac{y-x}{\sigma}$
is distributed as $\cN(0, I_d)$ and the power-mean inequality,
\begin{align*}
        \EE_{p_\sigma}\|\nabla \log p_\sigma(x) \|^p \leq \frac{1}{\sigma^p} \EE \left\|\frac{y-x}{\sigma} \right\|^p \lesssim \sqrt{\frac{pd}{\sigma^2}}.
\end{align*}
We complete the proof.
\end{proof}

\paragraph*{Proof of Lemma \ref{tech}(\ref{2.}).}
 Let $\PP_{[t_k',t]} $ and $\QQ_{[t_k',t]} $ denote be the path measure of $(\tilde{x}_s)_{t_k' \leq s \leq t}$ and $(\hat{y}_s)_{t_k' \leq s\leq t}$. For any $a \in \RR^d$ we have
 \begin{align*}
     \KL(\tilde{p}_{t|t_k'}(\cdot|a) \| \hat{q}_{t|t_k'}(\cdot|a) ) \leq \KL(\PP_{[t_k',t]}(\cdot|\tilde{x}_{t_k'} = a) \|\QQ_{[t_k',t]}(\cdot|\hat{y}_{t_k'} = a) ).
 \end{align*}
 Thus, it suffices to show
 \begin{align} \label{KL-limit}
 \lim_{t\to {t_k'}+} \KL(\PP_{[t_k',t]}(\cdot|\tilde{x}_{t_k'} = a) \|\QQ_{[t_k',t]}(\cdot|\hat{y}_{t_k'} = a) ) = 0
 \end{align}
 for a.e. $a\in \RR^d$. For this, we implement Girsanov change of measure on $\PP_{[t_k',t]}(\cdot|\tilde{x}_{t_k'} = a) $ and $\QQ_{[t_k',t]}(\cdot|\hat{y}_{t_k'} = a)$. If Novikov's condition holds for a.e. $a\in \RR^d$, Girsanov's theorem yields
 \begin{align*}
    \KL(\PP_{[t_k',t]}(\cdot|\tilde{x}_{t_k'} = a) \|\QQ_{[t_k',t]}(\cdot|\hat{y}_{t_k'} = a) ) = \EE\left[\int_{t_k'}^t  \|\nabla \log \tilde{p}(\tilde{x}_s) - s(x, t_{N-k}) \|^2   |\tilde{x}_{t_k} = a\right]
 \end{align*}
 for the exponential integrator scheme, or 
 \begin{align*}
   &   \KL(\PP_{[t_k',t]}(\cdot|\tilde{x}_{t_k'} = a) \|\QQ_{[t_k',t]}(\cdot|\hat{y}_{t_k'} = a) ) \\ &= \EE\left[\int_{t_k'}^t  \left\|\nabla \log \tilde{p}(\tilde{x}_s) - s(a, t_{N-k}) + \frac{1}{2}(\tilde{x}_s - a) \right\|^2   |\tilde{x}_{t_k'} = a\right]
 \end{align*}
 for the Euler-Maruyama scheme. Hence, \eqref{KL-limit} is obtained by the Monotone Convergence Theorem and we conclude the proof. Now we check the Novikov condition, which is given by
 \begin{align*}
     \EE\left[\exp\left(\frac{1}{2}\int_{t_k'}^t  \|\nabla \log \tilde{p}(\tilde{x}_s) - s(a, t_{N-k}) \|^2 \d s\right) \Big| \tilde{x}_{t_k'} = a\right] < 
     \infty\qquad \text{(exponential integrator)},
\end{align*}
 or 
  \begin{align*}
     \EE\left[\exp\left(\frac{1}{2}\int_{t_k'}^t  \left\|\nabla \log \tilde{p}(\tilde{x}_s) - s(a, t_{N-k}) + \frac{1}{2}(\tilde{x}_s - a) \right\|^2 \d s \right) \Big| \tilde{x}_{t_k'} = a\right]< \infty\, \text{(Euler-Maruyama)}.
\end{align*}
Hence, it suffices to prove that the following hold for a.e. $a\in \RR^d$ when $t-t_k'$ is sufficient small (recall that we only care about the limit $t \to t_k' +$):
\begin{align} \label{1}
    \EE\left[\exp\left(\int_{t_k'}^t \|\nabla \log \tilde{p}(\tilde{x}_s) \|^2 \d s  \right)  | \tilde{x}_{t_k'} = a  \right] &< \infty\\
 \label{2}
    \EE\left[\exp\left(\int_{t_k'}^t \|\tilde{x}_s -a \|^2 \d s \right) | \tilde{x}_{t_k'} = a  \right] &< \infty
\end{align}
In fact, by Lemma \ref{score-sub} we have $\|\nabla \log {p}_t({x}_t) \|_{\psi_2} \lesssim \sqrt{\frac{d}{\sigma_t^2}} $. Thus
 \begin{align*}
     \left\|\int_{t_k'}^t \|\nabla \log \tilde{p}_s(\tilde{x}_s) \|^2 \d s\right\|_{\psi_1} \leq  \int_{t_k'}^t \|\nabla \log \tilde{p}_s(\tilde{x}_s) \|_{\psi_2}^2 \d s \lesssim \frac{d}{\sigma_{T-t}^2}(t-t_k').
 \end{align*}
 When $t-t_k'$ is sufficient small, we have 
 \begin{align*}
       \left\|\int_{t_k'}^t \|\nabla \log \tilde{p}_s(\tilde{x}_s) \|^2 \right\|_{\psi_1} \leq \frac{1}{2},
 \end{align*}
 and thus
 \begin{align*}
  &   \EE_{\tilde{p}_{t_k'}(a)} \left[ \EE\left[\exp\left(\int_{t_k'}^t\|\nabla \log \tilde{p}(\tilde{x}_s) \|^2 \d s   \right) | \tilde{x}_{t_k'} = a  \right]\right]  \\
     & =\EE\left[\exp\left(\int_{t_k'}^t \|\nabla \log \tilde{p}(\tilde{x}_s) \|^2 \d s   \right)  \right] < \infty.
 \end{align*}
 Therefore, \eqref{1} holds for a.e. $a \in \RR^d$. To verify \eqref{2}, we split it as
 \begin{align}
     \int_{t_k'}^t \|\tilde{x}_s -a \|^2 \d s  \leq 2\int_{t_k'}^t\|\tilde{x}_s - \alpha_{T-s,T-t_k'}^{-1}a \|^2 +2(\alpha_{T-s,T-t_k'}^{-1} - 1)^2 (t-t_k')\| a \|^2. \label{3}
 \end{align}
The second term in the right hand side of \eqref{3} is a constant so we only need to consider the first term. Note that
\begin{align*}
    \left\|2\int_{t_k'}^t  \|\tilde{x}_s - \alpha_{T-s,T-t_k'}^{-1}\tilde{x}_{t_k'} \|^2 \d s\right\|_{\psi_1} & \leq 2\int_{t_k'}^t  \|\tilde{x}_s - \alpha_{T-s,T-t_k'}^{-1}\tilde{x}_{t_k'} \|_{\psi_2}^2 \d s\\
  & \lesssim 2(e^{-t+t_k'} - 1)(t-t_k').
\end{align*}
Thus when $t-t_k'$ is sufficient small we have
\begin{align*}
    \left\|2\int_{t_k'}^t \|\tilde{x}_s - \alpha_{T-s,T-t_k'}^{-1}\tilde{x}_{t_k'} \|^2\d s  \right\|_{\psi_1} \leq \frac{1}{2}
\end{align*}
and thus
\begin{align*}
& \EE_{\tilde{p}_{t_k'}(a)}\left[ \EE \left[ \exp\left(2\int_{t_k'}^t  \|\tilde{x}_s - \alpha_{T-s,T-t_k'}^{-1}a \|^2 \d s \right)| \tilde{x}_{t_k'} = a \right]  \right] \\ & =   \EE \exp\left(2\int_{t_k'}^t  \|\tilde{x}_s - \alpha_{T-s,T-t_k'}^{-1}a \|^2 \d s \right) <  \infty.
\end{align*}
We complete the proof of \eqref{2}.  \qed
\end{proof}


\end{document}